\documentclass{article}

\usepackage[main, final]{neurips_2026}

\usepackage[utf8]{inputenc} 
\usepackage[T1]{fontenc}    
\usepackage{hyperref}       
\usepackage{url}            
\usepackage{booktabs}       
\usepackage{amsfonts}       
\usepackage{nicefrac}       
\usepackage{microtype}      
\usepackage{xcolor}         

\usepackage{enumitem}
\usepackage{etoolbox}
\usepackage{array}
\usepackage[table]{xcolor}
\usepackage{graphicx}
\usepackage{wrapfig}
\usepackage{capt-of}
\usepackage{amsmath}
\usepackage{amssymb}
\usepackage{mathtools}

\usepackage{amsthm}
\newtheoremstyle{definition_styles}
  {\topsep}
  {\topsep}
  {\normalfont}
  {}
  {\bfseries}
  {.}
  {.5em}
  {\thmname{#1}\thmnumber{ #2}\thmnote{ (\bfseries #3)}}
\theoremstyle{definition_styles}

\newtheorem{theorem}{Theorem}
\newtheorem{lemma}{Lemma}
\newtheorem{assumption}{Assumption}

\newtheorem{remark}{Remark}
\newtheorem{corollary}{Corollary}

\makeatletter
\renewcommand\paragraph{\@startsection{paragraph}{4}{\z@}%
  {0.0ex \@plus 0.2ex \@minus 0.1ex}%
  {-1em}%
  {\normalfont\normalsize\bfseries}}
\renewenvironment{proof}[1][\proofname]{\par
  \pushQED{\qed}%
  \normalfont
  \topsep=0.6ex \@plus 0.1ex \@minus 0.1ex\relax
  \partopsep=\z@
  \trivlist
  \item[\hskip\labelsep\itshape #1\@addpunct{.}]\ignorespaces
}{%
  \popQED\endtrivlist\@endpefalse
}
\makeatother

\title{Learning Provable Neural Network Observer for Uncertain Dynamical Systems}

\newcommand{\affmark}[1]{\textsuperscript{#1}}

\makeatletter
\patchcmd{\@maketitle}{\rule{\z@}{24\p@}\@author}{\rule{\z@}{14\p@}\@author}{}{\PackageError{paper-layout}{Author padding adjustment failed}{}}
\patchcmd{\@maketitle}{\vskip 0.3in \@minus 0.1in}{\vskip 0.12in \@minus 0.04in}{}{\PackageError{paper-layout}{Affiliation padding adjustment failed}{}}
\makeatother

\author{%
	Zhangyi Wang\affmark{1} \quad
	Jiaxu Liu\affmark{2,}\thanks{Corresponding authors.} \quad
	Song Chen\affmark{3,*} \quad
	Chao Xu\affmark{1} \quad
	Shengze Cai\affmark{1,*} \\
	\normalfont\footnotesize \affmark{1}College of Control Science and Engineering, Zhejiang University, China \\
	\normalfont\footnotesize \affmark{2}School of Science, Huzhou Normal University, China \\
	\normalfont\footnotesize \affmark{3}Department of Mathematics, National University of Singapore, Singapore \\
}

\begin{document}

\maketitle

\begin{abstract}
	In many safety-critical applications, control of uncertain dynamical systems relies on observers that estimate states and external disturbances. Neural network observers can improve estimation accuracy, but certifying their Lyapunov stability via Linear Matrix Inequality (LMI) constraints leads to large-scale semidefinite programs (SDPs) that are difficult to solve for large networks. To overcome this scalability bottleneck, we propose a novel two-stage training framework for provably stable neural network observers. Our approach decouples the optimization into \textbf{a point-guided Lyapunov pre-training phase}, which rapidly achieves high estimation accuracy and local stability over sampled states, followed by \textbf{an LMI fine-tuning phase} that efficiently satisfies a strict global Lyapunov stability certificate. We provide formal theoretical guarantees for local stability radii and probabilistic coverage over a prescribed compact error-state domain under specified regularity and sampling assumptions. Experiments on nonlinear control benchmarks and X-29 aircraft ablations show that our LMI-certified neural network observers train significantly faster than direct LMI-based methods and generalize robustly across diverse systems, achieving improved tracking accuracy over a range of observer baselines. The code is available at \url{https://github.com/Berry-Myon/LearningNeuralNetworkObserver}.
\end{abstract}

\section{Introduction}

The control and observation of real-world dynamical systems are fundamentally challenged by the presence of external disturbances and unmodeled dynamics, which often appear as complex black-box effects in practical engineering \cite{BANASZUK201177, das2025dronediffusion}. This challenge is particularly critical in safety-critical domains such as plasma magnetic control \cite{de2019plasma, tartaglione2022plasma} and high-precision trajectory tracking for autonomous aerial vehicles \cite{shi2019neural, tian2023high, alqudsi2024advanced}. As demonstrated in recent studies on Neural Lander \cite{shi2019neural}, accurately estimating state-dependent uncertainties is essential for maintaining stability in unpredictable environments. Therefore, developing robust state and disturbance estimation techniques is critical for safely tracking desired trajectories and maintaining overall system stability in such environments \cite{alan2022disturbance, IZADI2024105854, wang2023disturbance}.

In uncertain dynamical systems, many observer designs for state estimation have been developed with rigorous theoretical guarantees. Classical approaches such as Sliding Mode Observers (SMO) \cite{xiong2001sliding, tan2001lmi, shtessel2014sliding}, Extended State Observers (ESO) \cite{talole2009extended, guo2011ESO, pu2015class}, and $H_\infty$-based filters \cite{karimi2009robust, zhao2018decentralized, zhao2019theoretical} provide robustness against disturbances and modeling errors, but typically rely on linear approximations or specific structural assumptions on the uncertainty. As a result, their performance degrades when faced with strong nonlinearities or significant model-plant mismatch. In parallel, machine learning has increasingly been applied to estimation and control \cite{deisenroth2011pilco, berkenkamp2017safe, chow2018lyapunov, hewing2020learning}, with neural network observers emerging as a powerful tool for capturing complex, unstructured uncertainties \cite{chen2017neural, hou2022neural, pmlr-v235-yang24f}. Most relevant to our setting, \cite{chen2023neural} introduced a neural observer for uncertain nonlinear systems that augments the observer dynamics with a neural network term and derives a Lyapunov stability certificate through a Linear Matrix Inequality (LMI) condition on the network parameters, thereby showing that expressive neural observers can be made provably stable. Complementary analyses based on quadratic constraints and semidefinite programming (SDP) relaxations have provided tractable stability conditions for dynamical systems with neural network controllers \cite{yin2021stability, Fazlyab2019Lipschitz}, establishing an important analytical foundation for certifying neural architectures in feedback loops. \cite{dawson2023safe} further surveyed neural Lyapunov, barrier, and contraction certificate methods, emphasizing both the promise of learned certificates and the practical difficulty of scaling formal guarantees to complex neural models. However, current neural network controller and observer frameworks face a critical computational bottleneck: certifying Lyapunov stability typically requires solving LMIs \cite{https://doi.org/10.1002/rnc.4528, Fazlyab2019Lipschitz, Pauli2022NeuralNT}. For high-dimensional systems or large-scale networks, the resulting SDP tasks become extremely slow and often numerically intractable \cite{doi:10.1137/110825844, Fazlyab2019Lipschitz, Pauli2022NeuralNT}. 

In response, related machine learning work has pursued direct Lyapunov-guided learning of certified neural controllers \cite{chang2019neural, zhou2022unknown, wu2023discrete}, while recent neural observer approaches often avoid direct end-to-end certificate-constrained optimization and instead rely on falsification, mixed-integer or satisfiability modulo theories (SMT) verification, or post-hoc certification \cite{pmlr-v235-yang24f}. Inspired by the expressive power of deep learning and the need for scalable certification, this paper proposes a new framework to resolve this conflict.

\paragraph{Main Contributions.}
Our work builds on classical Lyapunov stability analysis and LMI-based certification techniques for nonlinear dynamical systems, and develops a general training framework for large-scale, LMI-constrained learning architectures. Specializing this framework to neural network observers for uncertain systems, our main contributions are:
\begin{itemize}[leftmargin=1em,topsep=4pt,itemsep=4pt,parsep=0pt,partopsep=0pt]
	\item \textbf{A scalable two-stage training framework (Section \ref{sec-method}).} We propose a two-stage pipeline that decouples expressive observer learning from global stability certification: Stage~I performs point-guided Lyapunov pre-training over sampled error states to quickly learn a high-capacity neural network observer with strong empirical estimation performance, while Stage~II starts from the above warm initialization and applies a lightweight LMI fine-tuning step to recover a strict global Lyapunov stability certificate. By avoiding full-scale LMI-constrained optimization throughout training, this framework makes LMI-certified neural network observers practical for deep, high-capacity architectures beyond the reach of standard SDP-based methods.
	\item \textbf{Point-guided Lyapunov pre-training with theoretical characterization (Section \ref{sec-analysis}).} We introduce a sampling-based Lyapunov loss that enforces strict Lyapunov decrease at sampled error states, thereby giving Stage~I a stability-oriented objective that provides a Lyapunov-based warm start for LMI certification. Under mild Lipschitz assumptions on the observer nonlinearity, we show that this sampled decrease condition induces explicit local stability neighborhoods around each sample and yields a probabilistic coverage interpretation over a prescribed compact error-state domain under a sampling distribution with full support on that domain. These results characterize the local stability structure promoted by point-guided pre-training before the final global LMI certification step (Theorems \ref{thm-local_radius} and \ref{thm-pos}).
	\item \textbf{Comprehensive empirical validation (Section \ref{sec-experiments} and Appendix~\ref{sec-exp1}).} We evaluate the proposed framework on a quadrotor unmanned aerial vehicle (Quad-UAV) under ground effect and an autonomous underwater vehicle (AUV) subject to complex fluid disturbances, comparing against data-driven and classical ESO-based baselines. We report X-29 aircraft training-strategy ablations in Section~\ref{sec-ablation} and additional capacity and robustness results in Appendix~\ref{sec-exp1}. Across these settings, the final trained observers retain formal stability certificates while improving tracking and disturbance rejection.
\end{itemize}

\section{Background and Motivation}\label{sec-pro}

This section introduces the uncertain-system setting, the neural network observer studied in this paper, and the Lyapunov stability-certification bottleneck that motivates our two-stage framework.

Consider a general uncertain nonlinear dynamical system subject to external disturbances and unmodeled dynamics, represented as follows:
\begin{equation}
	\begin{cases}
		\dot{x}(t) = Ax(t) + Bu(t) + B_\omega \mathcal{K}(x, d, t), \\
		y(t) = Cx(t),
	\end{cases}
\end{equation}
where $x(t) \in \mathbb{R}^{n_{in}}$ is the system state, $u(t) \in \mathbb{R}^{n_u}$ is the control input, $y(t) \in \mathbb{R}^{n_{out}}$ is the measured output, $\mathcal{K}(x, d, t)$ encapsulates both the nonlinear dynamics and the external disturbances, and the system matrices $A, B, C, B_\omega$ are known and fixed. Since the internal state is often difficult to measure directly, state observation is essential in uncertain systems, and many control laws are built on the observed state. For this reason, the observer is designed to provide accurate state and disturbance estimation so that the closed-loop system can track a desired trajectory $x_d(t)$ under the uncertainty induced by $\mathcal{K}(x, d, t)$.

Following \cite{guo2011ESO, han2009ADRC, chen2023neural}, a neural network observer is designed for augmented-state estimation using the output estimation error as feedback:
\begin{equation}
	\begin{cases}
		\dot{\hat{x}}_1(t) = A\hat{x}_1(t) + Bu(t) + B_\omega \hat{x}_2(t) + \pi_{\theta_1}(\epsilon^{-1}(y(t) - \hat{y}(t))), \\
		\dot{\hat{x}}_2(t) = \epsilon^{-1} \pi_{\theta_2}(\epsilon^{-1}(y(t) - \hat{y}(t))),                                        \\
		\hat{y}(t) = C\hat{x}_1(t),
	\end{cases}
\end{equation}
where $\hat{x}_1(t)$ and $\hat{x}_2(t)$ denote the observer outputs for state estimation and disturbance estimation, respectively. The function $\pi_\theta = [\pi_{\theta_1}^\top, \pi_{\theta_2}^\top]^\top$ is a neural network parameterized by $\theta$, and $\epsilon > 0$ is a tunable scaling factor. A concrete residual-network parameterization of $\pi_\theta$ is provided in Appendix~\ref{apd-observer-param}.

Defining the scaled state estimation error and the lumped disturbance estimation error as $\eta_1 = \epsilon^{-1} (x - \hat{x}_1)$ and $\eta_2 = \mathcal{K}(x,d,t) - \hat{x}_2$, respectively, the observer induces the error dynamics
\begin{equation}
	\dot{\eta}(t) = A_\epsilon \eta(t) - \nu(\eta, \theta) + \epsilon w(t),
\end{equation}
where $\eta = [\eta_1^\top, \eta_2^\top]^\top$, $A_\epsilon = \begin{bmatrix} \epsilon A & B_\omega \\ O & O \end{bmatrix}$, $\nu^\top(\eta, \theta) = \begin{bmatrix}
		\pi_{\theta_1}(\eta_1) \\
		\pi_{\theta_2}(\eta_1)
	\end{bmatrix}$, and $w(t) \triangleq \frac{d}{dt}\mathcal{K}(x,d,t)$ denotes the time derivative of the lumped disturbance term. This architecture allows the observer to capture high-frequency uncertainties that linear Luenberger-type observers fail to model.

Given the observer error dynamics above, the next question is how to guarantee that the estimation errors decay despite the unknown disturbance term. A natural route is to study a quadratic Lyapunov function for the error system and derive a sufficient decrease condition that can be enforced during training. Specifically, by differentiating $V(\eta)=\eta^\top P\eta$ along the error dynamics and using the disturbance bound, we obtain the pointwise inequality in Eq.~(\ref{eqn-Point}), leading to the following criterion.

\begin{lemma}[\cite{khalil2002nonlinear}]\label{thm-lyapunov}
	Consider the quadratic Lyapunov candidate $V(\eta) = \eta^\top P \eta$ with a symmetric positive definite matrix $P \succ O$. Assume that the lumped disturbance derivative is bounded such that $\|w(t)\| \le M_w$. If there exists a margin $\rho > 0$ such that
	\begin{equation}\label{eqn-Point}
		\eta^\top(A_\epsilon^\top P + P A_\epsilon + \rho P)\eta
		- 2\nu^\top(\eta, \theta) P \eta
		+ 2\epsilon M_w \|\eta^\top P\|
		\le 0,
	\end{equation}
	holds for all $\eta$ in the considered domain, then the time derivative of $V$ satisfies $\dot{V} \le -\rho V$, establishing the desired Lyapunov stability of the error system.
\end{lemma}

Lemma~\ref{thm-lyapunov} shows that the desired Lyapunov stability property can be guaranteed if the pointwise inequality in Eq.~(\ref{eqn-Point}) holds over the domain of interest. However, directly checking this condition for all $\eta$ is difficult because experiments or sampled evaluations cannot exhaust a continuous domain. Prior work therefore replaces the pointwise condition with a sufficient global matrix inequality in the network parameters, stated next.

\begin{lemma}[\cite{chen2023neural}]\label{thm-LMI}
	The Lyapunov stability condition in Lemma \ref{thm-lyapunov} can be enforced via an LMI
	\begin{equation}\label{eqn-LMI}
		\mathcal{H}(\theta) \triangleq \hat{R}_{\pi}^\top(\theta)
		\begin{bmatrix}
			A_\epsilon^\top P + P A_\epsilon & -P \\
			-P                                     & O
		\end{bmatrix}
		\hat{R}_{\pi}(\theta) + \hat{R}_{\xi}^\top(\theta)\, \Psi^\top M \Psi\, \hat{R}_{\xi}(\theta) \prec O,
	\end{equation}
	where matrices $\hat{R}_{\pi}(\theta)$ and $\hat{R}_{\xi}(\theta)$ are constructed from the neural network weights as detailed in Appendix~\ref{apd-observer-param}, $\Psi$ is the sector-boundedness matrix associated with the activation function, and $M = \begin{bmatrix} O & \Lambda \\ \Lambda & O \end{bmatrix}$ is the gain matrix with $\Lambda$ being a positive definite diagonal matrix. For the \texttt{tanh} activation, $\Psi = \begin{bmatrix} I & -I \\ O & I \end{bmatrix}$.
\end{lemma}

Therefore, adopting the above neural network observer for observer-based control and tracking design inevitably requires solving the LMI in Eq.~(\ref{eqn-LMI}) to satisfy the condition in Lemma~\ref{thm-LMI}. Although this condition provides a rigorous global Lyapunov stability certificate, solving it for high-capacity neural network observers is extremely challenging in practice.

\textbf{The difficulty comes from three related aspects.} \emph{First}, the estimation performance of neural network observers depends strongly on the expressive power of $\pi_\theta$: in complex aerospace or fluid dynamics systems, small-scale networks often fail to capture high-frequency uncertainties, whereas increasing network width and depth rapidly enlarges the parameter dimension of $\theta$. \emph{Second}, standard LMI solvers such as interior-point methods scale poorly with the size of the matrix variables, and in Eq.~(\ref{eqn-LMI}) the dimension of $\mathcal{H}(\theta)$ grows directly with the number of neurons and layers, making the resulting SDP computationally prohibitive for high-capacity networks. \emph{Third}, enforcing Eq.~(\ref{eqn-LMI}) directly during gradient-based training is itself numerically unstable and slow, because it turns the optimization over $\theta$ into a highly non-convex constrained problem.

\paragraph{Research Objective.} How can we train a \emph{high-capacity neural network observer} that still satisfies the \emph{global LMI certificate} in Lemma~\ref{thm-LMI}, without paying the \emph{prohibitive cost of solving large-scale SDPs} throughout training? This tension between expressive estimation and scalable certification is exactly what motivates the two-stage strategy developed in Section~\ref{sec-results}. Once the LMI-constrained Lyapunov certification problem for neural network observers is recast from large-scale SDP solving into a trainable optimization problem, the same framework can be readily extended to broader neural-network-based control systems that require scalable stability certificates.

\section{Main Results}\label{sec-results}

In this section, we first present the proposed two-stage training framework for scalable, stability-certified neural network observers, and then provide the associated Lyapunov-based convergence and stability analysis.

\subsection{Method Development}\label{sec-method}

The core idea is to decouple \emph{expressive observer learning} from \emph{global stability certification}. We first use point-guided pre-training, then optimize under the full LMI constraint from this initialization.

\begin{figure}[t]
	\centering
	\includegraphics[width=0.95\linewidth]{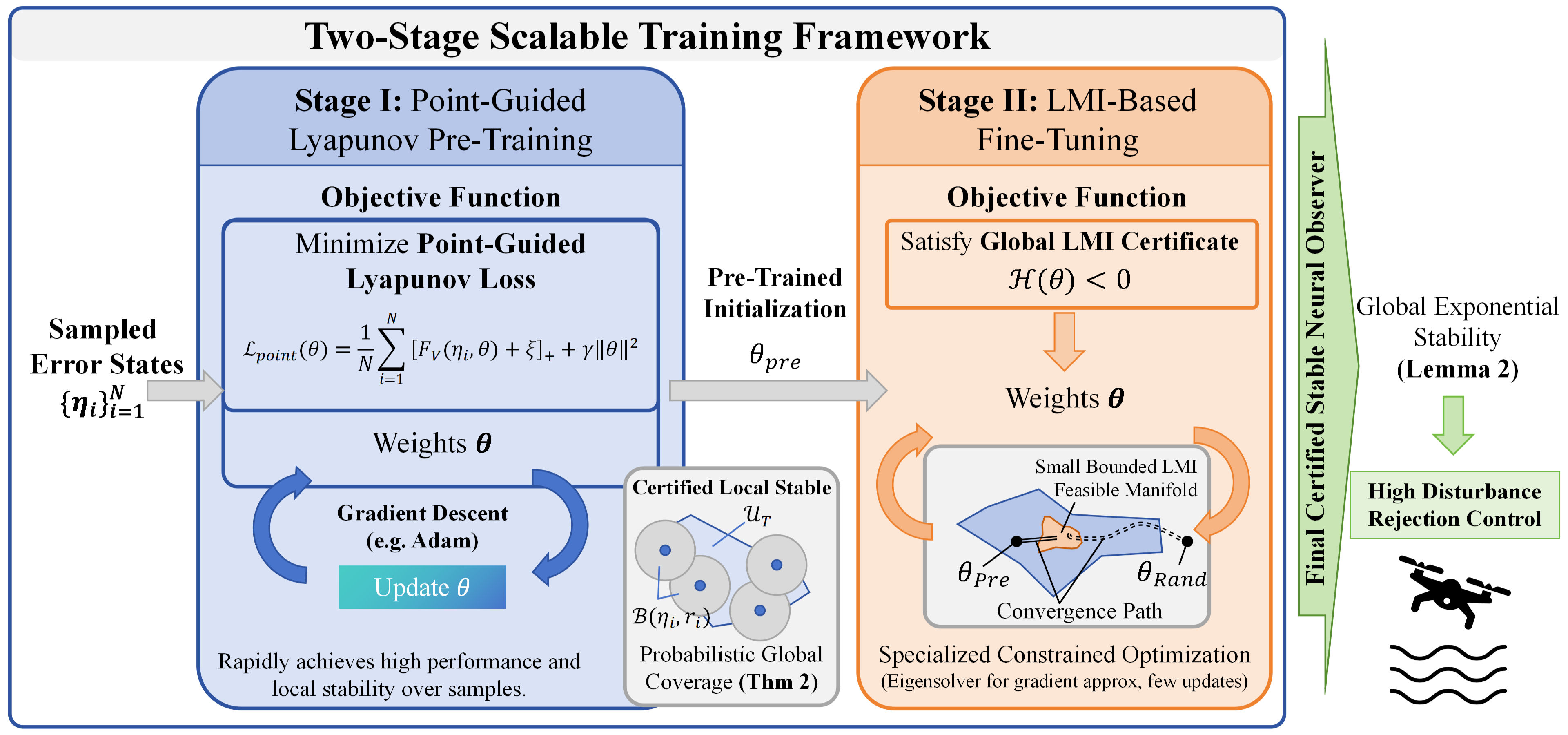}
	\caption{\textbf{Two-stage training framework for neural network observers.} Stage~I uses sampled error states for point-guided Lyapunov pre-training, and Stage~II applies lightweight LMI fine-tuning to obtain an observer satisfying the global LMI-based Lyapunov stability certificate.}
	\label{fig-Frm}
\end{figure}

\paragraph{Stage I: Point-Guided Lyapunov Pre-Training.}

The first stage learns a high-capacity observer that is locally stable around many sampled error states, without solving a large-scale LMI or SDP.

Let $\mathcal{X}$ be a prescribed compact domain in the error state space and let $\mathcal{D} = \{\eta_i\}_{i=1}^N$ be error states sampled from $\mathcal{X}$ (e.g., from a truncated Gaussian or uniform distribution). We consider the quadratic Lyapunov candidate $V(\eta) = \eta^\top P \eta$ with $P \succ O$ and define the computable Lyapunov upper bound
\begin{equation}
	\begin{aligned}
		F_V(\eta, \theta)
		 & \triangleq \eta^\top(A_\epsilon^\top P + P A_\epsilon + \rho P)\eta
		- 2\nu(\eta, \theta)^\top P\eta
		+ 2\epsilon M_w\|P\eta\|,
	\end{aligned}
\end{equation}
which satisfies $\dot{V}(\eta) + \rho V(\eta) \le F_V(\eta,\theta)$ by Lemma~\ref{thm-lyapunov}. Here, $\rho > 0$ is a desired decay rate and $M_w$ bounds the lumped disturbance derivative. This upper bound suggests a surrogate objective for Stage~I that directly penalizes violations of the Lyapunov decrease condition at sampled error states. These pointwise evaluations reduce the computational cost associated with enforcing the global LMI throughout high-capacity network training. In this way, the training objective already encourages the observer toward a stability-consistent region while preserving the efficiency and flexibility of standard gradient-based optimization. Based on this idea, we introduce the \emph{point-guided Lyapunov loss}
\begin{equation}\label{eq:point-loss}
	\mathcal{L}_{\mathrm{point}}(\theta)
	= \frac{1}{N} \sum_{i=1}^N \big[ F_V(\eta_i, \theta) + \xi \big]_+
	+ \gamma \|\theta\|^2,
\end{equation}
where $[\cdot]_+ = \max(0, \cdot)$ is the ReLU operation, $\xi > 0$ is a strict Lyapunov margin, and $\gamma > 0$ is a regularization coefficient. As shown formally in Lemma~\ref{lem-local}, once the data-dependent violation term vanishes, the Lyapunov decrease condition is strictly satisfied at all sampled points.

In Stage~I, illustrated on the left of Figure~\ref{fig-Frm}, we sample error states $\{\eta_i\}$ from $\mathcal{X}$ and minimize $\mathcal{L}_{\mathrm{point}}(\theta)$ by stochastic gradient descent to obtain a pre-trained parameter $\theta_{\mathrm{pre}}$. The purpose of this pre-training stage is to cheaply inject Lyapunov structure into the observer before imposing the full certificate: it quickly pushes the parameters away from clearly unstable regions and toward a near-feasible one using only pointwise evaluations. As a result, Stage~I preserves the scalability of unconstrained neural network training, exploits the expressive power of deep architectures, and provides a much better warm start for the subsequent LMI-based refinement than random initialization, all without solving any SDP.

\paragraph{Stage II: LMI-Based Fine-Tuning.}

Stage~I promotes stability over sampled states, and Stage~II enforces the global LMI condition in Lemma~\ref{thm-LMI}. The second stage, depicted on the right of Figure~\ref{fig-Frm}, therefore starts from $\theta_{\mathrm{pre}}$ and performs a fine-tuning phase to strictly satisfy the LMI $\mathcal{H}(\theta) \prec O$, thereby recovering the global LMI-based Lyapunov certificate from Lemma~\ref{thm-LMI} for the error system. In the LMI formulation of \cite{chen2023neural}, once this condition is satisfied, the observation error approaches zero as the gain $\epsilon$ is chosen sufficiently small.
We formulate the fine-tuning problem as a penalty-based optimization initialized at $\theta_{\mathrm{pre}}$:
\begin{equation}
	\min_{\theta}\ \mathcal{L}_{\mathrm{LMI}}(\theta)
	= \phi\big(\lambda_{\max}(\mathcal{H}(\theta))\big)
	+ \beta \|\theta\|^2,
\end{equation}
where $\lambda_{\max}(\cdot)$ denotes the largest eigenvalue, $\phi(\cdot)$ is a monotonically increasing penalty function that heavily penalizes positive eigenvalues (i.e., LMI violations), and $\beta>0$ is a regularization coefficient.

In Stage~II, we initialize at $\theta_{\mathrm{pre}}$ and fine-tune the network with the LMI penalty until $\mathcal{H}(\theta) \prec O$. Empirically, Stage~I often places the parameters closer to the feasible region than random initialization, which can reduce the number of updates needed in this stage.

\paragraph{Computational Advantages and Overall Pipeline.}

Our two-stage framework, summarized schematically in Figure~\ref{fig-Frm}, directly addresses the main scalability bottleneck of prior LMI-based neural network observer training. Stage~I uses an inexpensive, sample-based Lyapunov loss to quickly shape the parameters into a stability-oriented region, thereby exploiting the expressive power of deep networks without solving any SDPs. Stage~II then starts from $\theta_{\mathrm{pre}}$ and optimizes the LMI penalty to obtain a strictly feasible $\theta$. As shown empirically in Section~\ref{sec-ablation}, this warm start can reduce the practical optimization burden of the final certification step.

\subsection{Convergence and Stability Analysis}\label{sec-analysis}

We now analyze the Lyapunov structure induced by the two-stage method. The analysis proceeds in three steps to match the sample-based design of Stage~I. \emph{First,} Lemma~\ref{lem-local} establishes the direct consequence of minimizing the point-guided loss: once the sample-wise violation term vanishes, strict Lyapunov decrease holds at all sampled points. \emph{Second,} Theorem~\ref{thm-local_radius} and Corollary~\ref{cor-rad} extend these pointwise guarantees to the regions between samples by showing that each successfully trained sample induces a nontrivial certified neighborhood with a strictly positive lower bound on the local stability radius. \emph{Third,} Theorem~\ref{thm-pos} gives a probabilistic coverage interpretation of these local regions over the prescribed compact error-state domain under a sampling distribution with full support. The analysis characterizes Stage~I geometry. Finite-time optimization and sample-complexity guarantees for training success remain open. Finally, we connect the point-guided loss with the LMI certificate and explain why Stage~I can provide a useful initialization for Stage~II in practice.

\subsubsection{Point-Guided Lyapunov Optimization}

According to Lemma~\ref{thm-lyapunov}, if $F_V(\eta,\theta) \le 0$ holds for all $\eta$ in the considered domain, then the time derivative of $V(\eta)$ satisfies $\dot{V} \le -\rho V$, establishing the desired Lyapunov stability of the error system. The point-guided loss in Eq.~(\ref{eq:point-loss}) enforces a strict version of this condition at sampled points.

\begin{lemma}\label{lem-local}
	Let $\theta^*$ be any parameter such that the data-dependent term of $\mathcal{L}_{\mathrm{point}}$ vanishes, i.e., $\frac{1}{N} \sum_{i=1}^N \big[ F_V(\eta_i, \theta^*) + \xi \big]_+ = 0$. Then, for every sampled point $\eta_i \in \mathcal{D}$, the Lyapunov decrease condition is strictly satisfied: $F_V(\eta_i, \theta^*) \le -\xi < 0$.
\end{lemma}
\begin{proof}
    The proof is provided in Appendix~\ref{prf-lem-local}.
\end{proof}

\subsubsection{Local Stability Radius and Probabilistic Coverage}

Lemma~\ref{lem-local} establishes Lyapunov decrease at sampled points. Since neural networks define continuous mappings, this property extends to a neighborhood under mild regularity conditions.

\begin{assumption}[Lipschitz Continuity]\label{asp-lip}
	The neural network $\nu(\eta, \theta)$ is Lipschitz continuous with respect to $\eta$ on the compact domain $\mathcal{X}$, with Lipschitz constant $L_\nu$, i.e., $\|\nu(\eta_1) - \nu(\eta_2)\| \le L_\nu \|\eta_1 - \eta_2\|$ for all $\eta_1,\eta_2 \in \mathcal{X}$.
\end{assumption}

\begin{remark}
	Assumption~\ref{asp-lip} is used in the analysis of nonlinear observers and learning-based controllers. See, e.g., \cite{khalil2002nonlinear, chen2023neural, Fazlyab2019Lipschitz, pmlr-v242-zhang24a}. Our analysis uses global Lipschitz continuity of the observer nonlinearity on $\mathcal{X}$, with no Lipschitz smoothness requirement. This regularity condition applies to a broad class of practical observer nonlinearities.
\end{remark}

Under Assumption~\ref{asp-lip}, we obtain the following explicit local stability radius.

\begin{theorem}\label{thm-local_radius}
	Let $Q = A_\epsilon^\top P + P A_\epsilon + \rho P$ and fix the trained parameter $\theta$. Under Assumption~\ref{asp-lip}, suppose the Lyapunov condition is strictly satisfied at a sampled point $\eta_s$ such that $F_V(\eta_s,\theta) \le -\xi < 0$. Then there exists a continuous stability neighborhood $\mathcal{B}(\eta_s, r_s) \cap \mathcal{X}$ within which $F_V(\eta,\theta) \le 0$. When $C_2>0$, the radius $r_s$ can be chosen as the positive root of the quadratic equation:
	\begin{equation}
		C_2 r_s^2 + C_1(\eta_s) r_s + F_V(\eta_s,\theta) = 0,
	\end{equation}
	where $C_1(\eta_s)$ and $C_2$ are nonnegative coefficients whose explicit expressions, including the degenerate case $C_2=0$, are provided in Appendix~\ref{prf-local_radius}.
\end{theorem}
\begin{proof}
	The proof is provided in Appendix~\ref{prf-local_radius}.
\end{proof}

Moreover, these radii admit a strictly positive uniform lower bound.

\begin{corollary}\label{cor-rad}
	Assume $F_V(\eta_s,\theta) \le -\xi < 0$ for all sampled points $\eta_s \in \mathcal{D} \subset \mathcal{X}$ and Assumption~\ref{asp-lip} holds. Then there exists a constant $\underline{r} > 0$, independent of $\eta_s$, such that $r_{\min} \coloneqq \min_{1 \le s \le N} r_s \ge \underline{r} > 0$.
\end{corollary}
\begin{proof}
	The explicit expression of $\underline{r}$ and its derivation are provided in Appendix~\ref{prf-rad}.
\end{proof}

These results allow us to interpret the domain-level effect of point-wise training. We call a sampled point $\eta_i$ \emph{successful} if the trained parameter $\theta$ satisfies the strict sampled Lyapunov condition $F_V(\eta_i,\theta) \le -\xi$. Let $U_T = \bigcup_{i=1}^T \big(\mathcal{B}(\eta_i, r_i)\cap\mathcal{X}\big)$ denote the union of the certified local stable regions generated by $T$ successful samples inside the prescribed domain.

\begin{theorem}\label{thm-pos}
	Let $\mathcal{X} \subset \mathbb{R}^{n_{\eta}}$ be a compact domain in the error-state space of $\eta$, where $n_{\eta} = \dim(\eta)$. Let $\mu$ be a probability measure supported on $\mathcal{X}$ such that, for every $z \in \mathcal{X}$ and every $\varrho>0$, $\mu(\mathcal{B}(z,\varrho)\cap\mathcal{X})>0$. Assume $\eta_1,\dots,\eta_T$ are independent successful samples drawn from $\mu$, and each sample admits a certified neighborhood radius $r_i \ge r_{\min} > 0$. Then there exist constants $N_{\mathrm{cov}} < \infty$ and $q>0$, depending only on $\mathcal{X}$, $\mu$, and $r_{\min}$, such that
	\begin{equation}
		\mathbb{P}(\mathcal{X} \subset U_T) \ge 1 - N_{\mathrm{cov}}(1-q)^T.
	\end{equation}
	In particular, as the number of successful samples $T$ approaches infinity, the probability that the certified stable region $U_T$ completely covers the error-state domain $\mathcal{X}$ approaches $1$:
	\begin{equation}
		\lim_{T \to \infty} \mathbb{P}(\mathcal{X} \subset U_T) = 1.
	\end{equation}
\end{theorem}
\begin{proof}
	The proof is provided in Appendix~\ref{prf-pos}.
\end{proof}
A finite-sample consequence of Theorem~\ref{thm-pos} is provided in Corollary~\ref{cor-finite} in Appendix~\ref{prf-pos}.

\begin{remark}
	Theorems~\ref{thm-local_radius} and \ref{thm-pos} jointly show that enforcing a strict Lyapunov decrease at finitely many sampled error states induces nontrivial local stable neighborhoods around each sample and yields a probabilistic coverage interpretation over the prescribed compact error-state domain used in Stage~I. This can be viewed as a sample-driven expansion of a \emph{trusted} stable region for pre-training. Related learning-based control studies, e.g., \cite{chang2019neural, zhou2022unknown, pmlr-v242-zhang24a}, typically use stronger regularity assumptions such as Lipschitz smoothness. Our analysis derives explicit local radii and a finite-sample coverage bound using Lipschitz continuity of the observer nonlinearity and a full-support sampling assumption. This formulation leads to simpler and more practical regularity conditions. More broadly, this paradigm can also be extended to other stability-certified learning-for-control problems, such as neural network controller synthesis, e.g. \cite{gu2022recurrent, wu2023discrete, pmlr-v235-yang24f}. Appendix~\ref{apd-controller-extension} gives one example with recurrent neural network controllers.
\end{remark}

\begin{remark}
	Let $\Theta_{\mathrm{LMI}} = \{\theta: \mathcal{H}(\theta) \prec O\}$ and $\Theta_{\mathrm{point}} = \{\theta: F_V(\eta,\theta) < 0,\ \forall \eta \in \mathcal{X}\}$. Under the assumptions and certificate variables of Lemma~\ref{thm-LMI}, any $\theta \in \Theta_{\mathrm{LMI}}$ also belongs to $\Theta_{\mathrm{point}}$, so $\Theta_{\mathrm{LMI}} \subseteq \Theta_{\mathrm{point}}$. In practice, Stage~I provides a stability-oriented initialization by promoting pointwise Lyapunov decrease over sampled error states, which can move the parameters closer to the certified set. The global LMI is imposed during Stage~II, which starts from $\theta_{\mathrm{pre}}$ and empirically resolves the remaining LMI violations. Compared with random initialization, this warm start can reduce the practical cost of the final fine-tuning step.
\end{remark}

\section{Experiments}\label{sec-experiments}

In this section, we evaluate the proposed framework against representative observer and learning-based control baselines on two nonlinear benchmarks: a Quad-UAV under aerodynamic ground effect and an AUV navigating through dynamic fluid vortices. These experiments emphasize comparison with existing methods and the resulting gains in tracking and disturbance rejection. We provide the X-29 aircraft training-strategy ablations in Section~\ref{sec-ablation} and capacity and robustness results in Appendix~\ref{sec-exp1}. All experiments were conducted on our computing server. See Appendix~\ref{apd-expo} for details.

\subsection{X-29 Training Efficiency and Ablation Study}\label{sec-ablation}

We use the X-29 aircraft benchmark \cite{1992ntrs.rept09932B} to study training efficiency and the contribution of each training stage. The Large NN has 576 hidden neurons. The system matrices, observer architectures, and training settings are given in Appendix~\ref{apd-exp1}. Additional capacity and robustness results are reported in Appendix~\ref{sec-exp1}.

Table~\ref{tbl-time_use} reports the computational cost of enforcing LMI-based stability. Direct LMI solvers work only for tiny networks and become infeasible as capacity grows. For the Large NN, direct solving fails, while pure LMI gradient descent takes $2210.39$ seconds. The proposed two-stage method reduces this to $895.31$ seconds, giving a $2.5\times$ speedup while still satisfying the global LMI certificate.

\begin{table}[t]
	\centering
	\caption{\textbf{Computational efficiency and scalability of the two-stage training framework.} Comparison of solving or training time across observer capacities and optimization methods.}
	\label{tbl-time_use}
	\setlength{\tabcolsep}{4.5pt}
	\footnotesize
		\begin{tabular}{>{\bfseries}lclc}
			\toprule
			\textbf{Architecture}                 & \textbf{Total Neurons} & \textbf{Solving/Training Method}    & \textbf{Time (s)} \\
			\midrule
			Linear                                & 0                      & Pole Placement                      & <1                \\
			Tiny NN                               & 9                      & Direct LMI Solver                   & <1                \\
			Small NN                              & 128                    & Direct LMI Solver                   & Infeasible        \\
			Large NN                              & 576                    & Direct LMI Solver                   & Failed            \\
			Large NN                              & 576                    & LMI Gradient Descent (Only)         & 2210.39           \\
			\rowcolor{green!20} \textbf{Large NN} & \textbf{576}           & \textbf{Point \& LMI Tuning (Ours)} & \textbf{895.31}   \\
			\bottomrule
		\end{tabular}
\end{table}

\begin{table}[t]
	\centering
	\caption{\textbf{Ablation study on training strategy for the Large NN observer.} Comparison of training time and root mean squared estimation error (RMSE) for point-guided pre-training, pure LMI optimization, and the proposed two-stage Point \& LMI Tuning method.}
	\label{tbl-ablation}
	\setlength{\tabcolsep}{5pt}
	\footnotesize
	\begin{tabular}{>{\bfseries}lcc}
		\toprule
		\textbf{Training Method}                                & \textbf{Time (s)} & \textbf{RMSE (m)} \\
		\midrule
		Point-Guided (Only)                                     & 99.53             & 1.0317           \\
		LMI Gradient Descent (Only)                             & 2210.39           & 0.2659           \\
		\rowcolor{green!20} \textbf{Point \& LMI Tuning (Ours)} & \textbf{895.31}   & \textbf{0.2234}  \\
		\bottomrule
	\end{tabular}
\end{table}

\paragraph{Ablation Study: Effect of Point-Guided Pre-Training and LMI Fine-Tuning.}

We compare three training strategies for the Large NN: \emph{Point-Only}, \emph{LMI-Only}, and the proposed two-stage scheme.

\begin{figure}[t]
	\centering
	\includegraphics[width=0.8\linewidth]{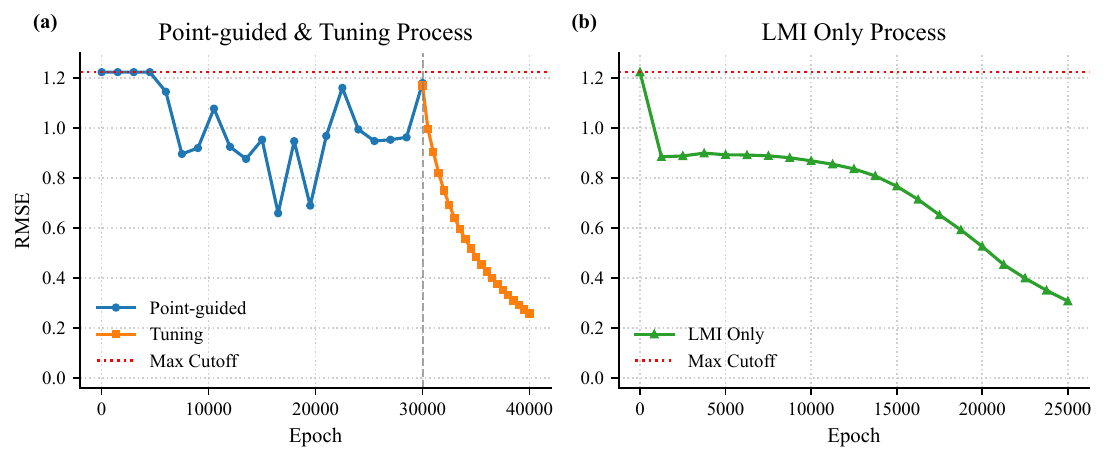}
	\caption{\textbf{Ablation on training strategy for the Large NN observer.} RMSE under \emph{Point-Only}, \emph{LMI-Only}, and the proposed two-stage scheme.}
	\label{fig-RMSE_CMP}
\end{figure}

Figure~\ref{fig-RMSE_CMP} and Table~\ref{tbl-ablation} show that Point-Only has the largest final RMSE and leaves LMI certification to Stage~II. Its pre-trained parameters provide a favorable initialization that substantially accelerates subsequent LMI fine-tuning. The two-stage method therefore achieves the best final RMSE while reducing training time from $2210.39$ seconds to $895.31$ seconds.

\subsection{Quad-UAV under Ground Effect}\label{sec-exp2}

We next evaluate the method on a Quad-UAV subject to aerodynamic ground effect, a nonlinear setting with strong state-dependent disturbances.

The 12-dimensional state is $x = [p, v, \Omega, \omega]^\top \in \mathbb{R}^{12}$, and the nonlinear dynamics with disturbance $f_a$ are
\begin{equation}
	\begin{aligned}
		\dot p         & = v,                  & \qquad m\dot v      & = m\vec{g}+R(f_u+f_a),             \\
		\dot R(\Omega) & = R(\Omega)S(\omega), & \qquad J\dot \omega & = J\omega\times\omega+\tau_\omega,
	\end{aligned}
\end{equation}
where $f_u$ and $\tau_\omega$ denote thrust and body torques, and $S(\cdot)$ denotes the skew-symmetric matrix operator. Near the ground, rotor downwash generates a cushioning force known as ground effect, modeled as 
\begin{equation}
    f_G = \frac{g_2 R(z)}{h^2 + g_1} T_{\mathrm{cmd}},
\end{equation}
where $h$ is the altitude, $R(z)$ is the $z$-axis component of the rotation matrix, $T_{\mathrm{cmd}}$ is the nominal thrust command, and $g_1, g_2$ are aerodynamic constants \cite{yang2025ground}. This term introduces strong nonlinear disturbances that degrade landing and take-off performance.

We compare three controllers: a \emph{Basic Proportional-Integral-Derivative (PID)} baseline, \emph{PID + Neural Lander} \cite{shi2019neural}, and \emph{PID + Neural Network Observer} (ours). Neural Lander learns a task-specific feedforward compensator from 40,000 data points, whereas our method estimates $f_a$ online and modifies the control thrust as $T_{\mathrm{cmd}}' = T_{\mathrm{cmd}} - \hat{f}_a$. Training details are provided in Appendix~\ref{apd-exp2}, with observer-error results in Appendix~\ref{apd-quad-observer-error}.

\begin{figure}[t]
	\centering
	\includegraphics[width=0.8\linewidth]{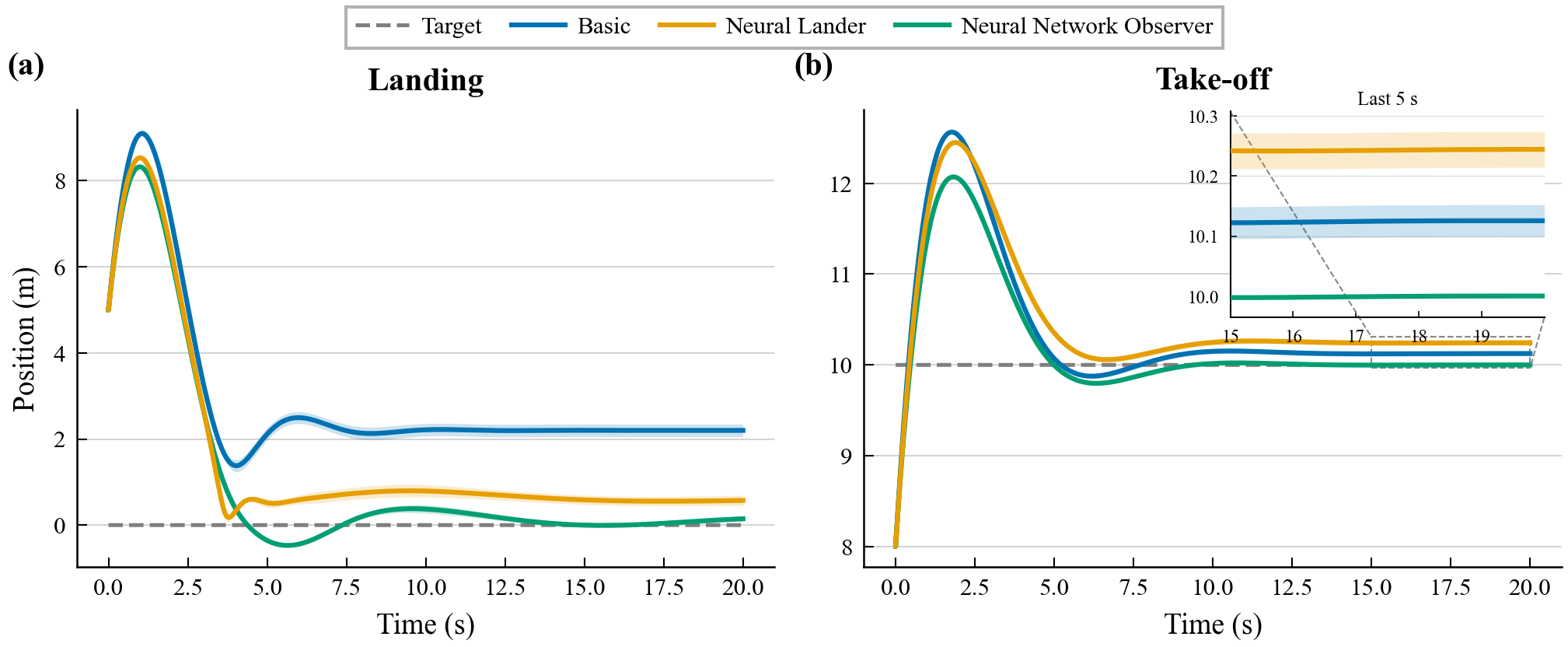}
	\caption{\textbf{Generalization performance on a Quad-UAV under ground effect.} Vertical trajectory tracking for landing (a) and take-off (b). Shaded regions denote the mean $\pm$ standard deviation (SD).}
	\label{fig-quad}
\end{figure}

\begin{table}
	\centering
	\caption{\textbf{Comparative analysis of tracking errors and data requirements for Quad-UAV control.} Here, ``per task'' means that once the environment or disturbance changes, new data must be collected and the model must be retrained. The label ``once'' means that the same trained parameters can be reused across changes in the environment as long as the underlying system is unchanged.}
	\label{tbl-quad}
	\setlength{\tabcolsep}{3pt}
	\footnotesize
	\begin{tabular}{>{\bfseries}lcccc}
		\toprule
		\textbf{Method}                                                   & \textbf{Data Required} & \textbf{Training Time}   & \textbf{Landing (m)}         & \textbf{Take-off (m)}        \\
		\midrule
		Basic PID                                                         & $\times$           & ---                      & $2.2006 \pm 0.1456$          & $0.1255 \pm 0.0266$          \\
		PID + Neural Lander (\cite{shi2019neural})                                              & \checkmark         & 88.2\,s (per task)       & $0.5732 \pm 0.1188$          & $0.2444 \pm 0.0297$          \\
		\rowcolor{green!20} \textbf{PID + Neural Network Observer (Ours)} & \boldmath $\times$ & \textbf{233.8\,s (once)} & $\mathbf{0.1471 \pm 0.0602}$ & $\mathbf{0.0002 \pm 0.0000}$ \\
		\bottomrule
	\end{tabular}
	\vspace{-1.1em}
\end{table}

Figure~\ref{fig-quad} and Table~\ref{tbl-quad} report vertical tracking performance for both landing and take-off. In landing, Basic PID incurs a large error of $2.2006$ m, while Neural Lander and our observer reduce it to $0.5732$ m and $0.1471$ m, respectively.

The take-off results further highlight generalization. Neural Lander degrades to $0.2444$ m error, worse than the Basic PID baseline ($0.1255$ m), whereas the proposed observer achieves near-perfect tracking with $0.0002$ m error after a single offline training stage. This suggests that the observer-based design generalizes across tasks more reliably than the task-specific baseline.

\subsection{Disturbance Rejection in Fluid Environment}\label{sec-exp3}

We further investigate the neural network observer in a highly complex and time-varying disturbance field, where an AUV navigates through a dynamic fluid environment. This scenario tests the ability of the method to handle strong, spatially and temporally varying disturbances arising from fluid-structure interactions.

The AUV is modeled as a non-holonomic agent confined to a 2D plane with state $x = [p_x, p_y, \psi, v, \omega]^\top$, representing 2D position, heading angle, linear velocity, and angular velocity. The dynamics are
\begin{equation}
	\begin{aligned}
		\dot{p}_x &= v\cos(\psi), \quad \dot{p}_y = v\sin(\psi), \quad \dot{\psi} = \omega, \\
		\dot{v}   &= \frac{F+f_a}{m}, \quad \dot{\omega} = \frac{\tau+\tau_a}{J},
	\end{aligned}
\end{equation}
where $u = [F, \tau]^\top$ is the control input, $m$ and $J$ are the mass and moment of inertia, and $f_a, \tau_a$ denote unknown fluid-induced forces and torques. The goal is to track a prescribed reference trajectory while rejecting these disturbances.

The surrounding fluid is simulated using the WaterLily engine on a $384 \times 192$ grid. The flow has an upstream velocity $U = 10.0$ and Reynolds number $Re = 200$. A fixed cylindrical obstacle of diameter $D_1 = 20$ is placed at $(144, 96)$, generating a von K\'arm\'an vortex street in its wake. The AUV, modeled as a $16 \times 8$ rectangular body, starts at $(128, 96)$ and must traverse the vortex street to reach a target at $(200, 130)$, experiencing strong, unsteady hydrodynamic forces along the way.

We compare three control configurations: a \emph{Basic Nonlinear Model Predictive Control (NMPC)} controller without explicit disturbance compensation, an \emph{ESO + NMPC} based on scheme \cite{guo2011ESO}, and the proposed \emph{Neural Network Observer + NMPC} framework. The NMPC implementation uses an interior-point optimizer \cite{wachter2006implementation}. All NMPC details and observer parameters are provided in Appendix~\ref{apd-exp3}, with observer-error results in Appendix~\ref{apd-water-observer-error}. Performance is evaluated via the overall trajectory tracking error (mean $\pm$ SD) over multiple runs with randomized initial conditions or disturbance realizations.

\begin{figure}[t]
	\centering
	\includegraphics[width=0.8\linewidth]{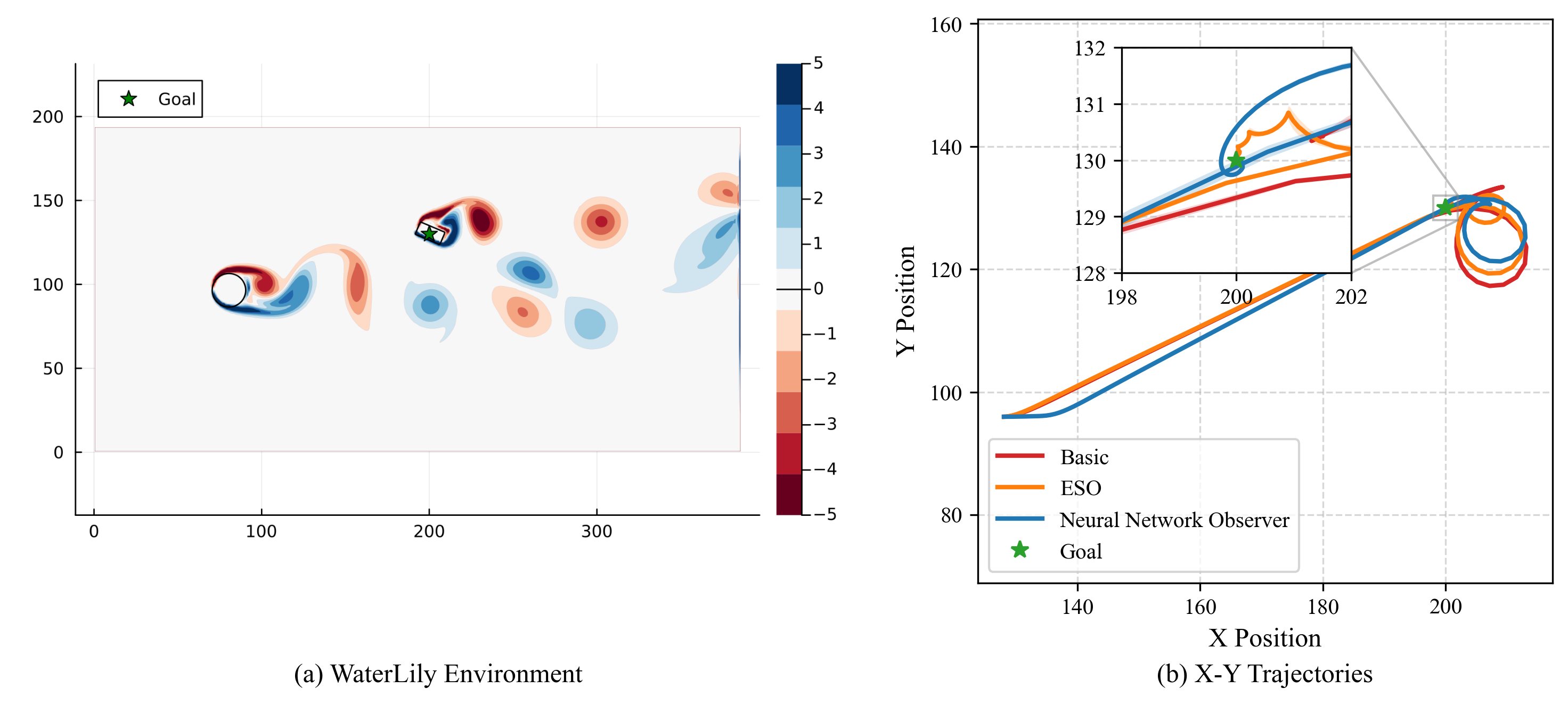}
	\caption{\textbf{AUV navigation in a simulated von K\'arm\'an vortex street.} (a) The WaterLily fluid environment showing pressure fields and the AUV's path. (b) Trajectory tracking results for different controllers. Shaded regions denote the mean $\pm$ SD.}
	\label{fig-AUV}
	\vspace{-0.2em}
\end{figure}

\begin{table}
	\centering
	\caption{\textbf{Disturbance rejection performance in dynamic fluid vortices.} Tracking error of different control configurations in the vortex environment.}
	\label{tbl-fluid}
	\setlength{\tabcolsep}{7pt}
	\footnotesize
	\begin{tabular}{lc}
		\toprule
		\textbf{Method}                                                    & \textbf{Tracking Error (m)} \\
		\midrule
		\textbf{Basic NMPC}                                                & $1.4932 \pm 0.0199$           \\
		\textbf{NMPC + ESO (\cite{guo2011ESO})}                                                & $0.0959 \pm 0.0088$           \\
		\rowcolor{green!20} \textbf{NMPC + Neural Network Observer (Ours)} & $\mathbf{0.0499 \pm 0.0019}$  \\
		\bottomrule
	\end{tabular}
	\vspace{-1.1em}
\end{table}

Table~\ref{tbl-fluid} shows that adding ESO reduces the Basic NMPC error from $1.4932 \pm 0.0199$ m to $0.0959 \pm 0.0088$ m, while the proposed Neural Network Observer + NMPC further lowers it to $0.0499 \pm 0.0019$ m, a $48.0\%$ reduction relative to ESO. Figure~\ref{fig-AUV} confirms this improvement qualitatively: the neural observer tracks the reference through the vortex street more closely, indicating stronger disturbance rejection than the classical ESO baseline.

\section{Conclusion}\label{sec-conclusion}

In this paper, we present a scalable, provably stable neural network observer framework for state and disturbance estimation in uncertain dynamical systems. To mitigate the computational intractability of enforcing global LMI constraints on high-capacity networks, we introduce a two-stage pipeline: point-guided Lyapunov pre-training followed by LMI fine-tuning. This approach permits the use of deep architectures while maintaining a final global LMI-based Lyapunov certificate. In addition, the Stage~I analysis provides explicit local stability radii and a probabilistic coverage interpretation over a prescribed compact error-state domain under a full-support sampling assumption. Evaluations on Quad-UAV and AUV benchmarks indicate improved robustness and generalization under the tested conditions compared to data-driven and classical ESO-based baselines, while the X-29 ablations in Section~\ref{sec-ablation} examine training efficiency and Appendix~\ref{sec-exp1} reports additional capacity and robustness results. By satisfying stability requirements offline without requiring per-task disturbance datasets, our framework provides a practical means for deploying deep learning methods in safety-critical control settings under model uncertainty. More broadly, the same training-and-certification principle can be extended to neural network systems whose Lyapunov stability guarantees are imposed through LMI constraints.

\newpage
\begin{ack}
	This work was supported by FAST-LAB, College of Control Science and Engineering, Zhejiang University.
\end{ack}

\bibliographystyle{abbrvnat}
\bibliography{reference}

\newpage

\appendix

\section{Limitations and Future Work}\label{sec-lim}

Our framework has several limitations. First, the Stage~I analysis is carried out on a prescribed compact domain $\mathcal{X}$. The local-radius and probabilistic-coverage results in Theorems~\ref{thm-local_radius} and \ref{thm-pos} characterize domain-dependent properties of pre-training. The scope of the final LMI certificate is determined separately by its certification assumptions. Finite-sample guarantees for point-guided training remain open, as do the required sample size and conditions for successful gradient-based certification.

Second, the method relies on a reasonably informative nominal model and on deployment conditions that are consistent with the offline certification assumptions. Significant model mismatch, severe noise, or out-of-distribution disturbances may therefore degrade both estimation accuracy and certificate usefulness. Finally, the two-stage pipeline improves scalability relative to direct LMI-constrained training. The offline certification step can remain costly for very large architectures or high-dimensional systems. The controller extension in Appendix~\ref{apd-controller-extension} suggests broader applicability under the assumptions of the underlying controller LMI certificate and requires dedicated closed-loop validation. Future work will focus on less conservative and more scalable certificates, finite-sample analysis for Stage~I, and deployment-aware mechanisms such as runtime monitoring, safe switching, and controller-side experiments.

\section{Structure of the Neural Network Observer}\label{apd-observer-param}

The neural term $\pi_\theta$ in the observer dynamics is implemented as a residual feedforward network, as illustrated in Figure~\ref{fig-resnet-param}. The network input is $x \in \mathbb{R}^{n_{\mathrm{in}}}$, and the network output is $\pi_\theta(x)\in\mathbb{R}^{n_{\mathrm{out}}}$. For an $L$-layer network with hidden widths $n_1,\dots,n_L$, the layer-wise propagation is defined as
\begin{equation}
	\pi_\theta^{[0]}(x)=x,\qquad
	\pi_\theta^{[l]}(x)=\sigma\!\left(W^{[l]}\pi_\theta^{[l-1]}(x)\right),\quad l=1,\dots,L,
\end{equation}
where $W^{[l]}\in\mathbb{R}^{n_l\times n_{l-1}}$, $n_0=n_{\mathrm{in}}$, and $\sigma(\cdot)$ is the activation function. The observer network output is
\begin{equation}
	\pi_\theta(x)=W^{[L+1]}\pi_\theta^{[L]}(x)+W^{[L+2]}\pi_\theta^{[0]}(x),
\end{equation}
where $W^{[L+1]}\in\mathbb{R}^{n_{\mathrm{out}}\times n_L}$ maps the final hidden layer to the output and $W^{[L+2]}\in\mathbb{R}^{n_{\mathrm{out}}\times n_{\mathrm{in}}}$ is the linear shortcut. The parameter set is
\begin{equation}
	\theta=\{L,n_1,\dots,n_L,W^{[1]},\dots,W^{[L+2]}\}.
\end{equation}
The output vector is partitioned according to the observer dynamics into $\pi_{\theta_1}$ and $\pi_{\theta_2}$.

The LMI lifting matrices $\hat{R}_{\pi}(\theta)$ and $\hat{R}_{\xi}(\theta)$ in Lemma~\ref{thm-LMI} follow the neural-network isolation construction in \cite{chen2023neural}. The stacked pre-activation vector and activation-output vector are
\begin{equation}
	\xi_{\sigma}(x)=
	\begin{bmatrix}
		W^{[1]}\pi_\theta^{[0]}(x) \\
		\vdots                     \\
		W^{[L]}\pi_\theta^{[L-1]}(x)
	\end{bmatrix},\qquad
	s_{\sigma}(x)=
	\begin{bmatrix}
		\pi_\theta^{[1]}(x) \\
		\vdots              \\
		\pi_\theta^{[L]}(x)
	\end{bmatrix}.
\end{equation}
With $n_{\sigma}=\sum_{l=1}^{L}n_l$, the residual network can be written as the linear relation
\begin{equation}
	\begin{bmatrix}
		\pi_\theta(x) \\
		\xi_{\sigma}(x)
	\end{bmatrix}
	=
	\begin{bmatrix}
		N_{\pi x}  & N_{\pi s}  \\
		N_{\xi x} & N_{\xi s}
	\end{bmatrix}
	\begin{bmatrix}
		x \\
		s_{\sigma}(x)
	\end{bmatrix},
\end{equation}
where
\begin{equation}
	\begin{aligned}
		N_{\pi x}  &= W^{[L+2]},\qquad
		N_{\pi s} = \begin{bmatrix} O & \cdots & O & W^{[L+1]} \end{bmatrix},\\
		N_{\xi x} &= \begin{bmatrix} (W^{[1]})^\top & O & \cdots & O \end{bmatrix}^\top,
	\end{aligned}
\end{equation}
and $N_{\xi s}\in\mathbb{R}^{n_{\sigma}\times n_{\sigma}}$ is the block lower-shift matrix whose $(l,l-1)$ block is $W^{[l]}$ for $l=2,\dots,L$ and whose other blocks are zero. Therefore,
\begin{equation}
	\hat{R}_{\pi}(\theta)=
	\begin{bmatrix}
		I        & O         \\
		N_{\pi x} & N_{\pi s}
	\end{bmatrix},\qquad
	\hat{R}_{\xi}(\theta)=
	\begin{bmatrix}
		N_{\xi x} & N_{\xi s} \\
		O         & I
	\end{bmatrix}.
\end{equation}
In Lemma~\ref{thm-LMI}, this construction is applied to the observer nonlinearity $\nu(\eta,\theta)$, whose network input is the corresponding error-feedback signal represented as a fixed linear map of $\eta$. This linear map is absorbed into $W^{[1]}$ and $W^{[L+2]}$ when forming the $N$ matrices. If $\pi_{\theta_1}$ and $\pi_{\theta_2}$ are implemented as separate subnetworks, the corresponding $N$ matrices are obtained by vertical and block-diagonal stacking of the subnetwork matrices.

\begin{figure}[t]
	\centering
	\includegraphics[width=0.85\linewidth]{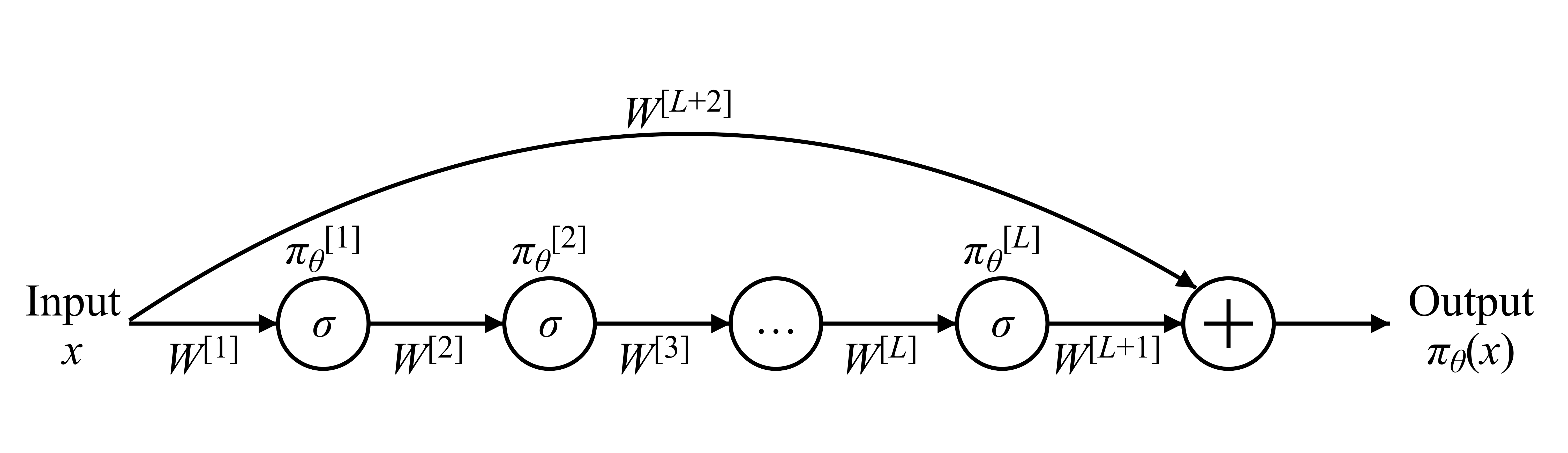}
	\caption{\textbf{Residual network parameterization for the neural observer.} The network uses hidden layers with the activation function $\sigma$ and a linear shortcut $W^{[L+2]}$ from the input to the output.}
	\label{fig-resnet-param}
\end{figure}

\section{Proof for Section \ref{sec-results}}\label{prf-all}

\subsection{Proof of Lemma \ref{lem-local}}\label{prf-lem-local}

\begin{proof}
	If the data-dependent term of the loss is zero, we have
	\begin{equation*}
		[F_V(\eta_i, \theta^*) + \xi]_+ = 0, \forall i \in \{1,\dots,N\}.
	\end{equation*}
	By the definition of the ReLU function, this requires
	\begin{equation*}
		F_V(\eta_i, \theta^*) + \xi \le 0,
	\end{equation*}
	and hence,
	\begin{equation*}
		F_V(\eta_i, \theta^*) \le -\xi.
	\end{equation*}
\end{proof}

\subsection{Proof of Theorem \ref{thm-local_radius}}\label{prf-local_radius}

\begin{proof}
	Let $\eta = \eta_s + \zeta \in \mathcal{X}$. Since $Q=A_\epsilon^\top P+P A_\epsilon+\rho P$ is symmetric, we evaluate the difference $F_V(\eta,\theta) - F_V(\eta_s,\theta)$ as follows:
	\begin{equation*}
		\begin{aligned}
			 & \quad F_V(\eta,\theta)-F_V(\eta_s,\theta)                                                                                          \\
			 & = \eta^\top Q\eta - \eta_s^\top Q\eta_s - 2\nu(\eta,\theta)^\top P\eta + 2\nu(\eta_s,\theta)^\top P\eta_s + 2\epsilon M_w(\|P\eta\| - \|P\eta_s\|) \\
			 & = 2\eta_s^\top Q\zeta + \zeta^\top Q\zeta - 2(\nu(\eta,\theta)-\nu(\eta_s,\theta))^\top P\eta_s - 2\nu(\eta_s,\theta)^\top P\zeta                      \\
			 & \quad - 2(\nu(\eta,\theta)-\nu(\eta_s,\theta))^\top P\zeta + 2\epsilon M_w(\|P\eta\| - \|P\eta_s\|)                                                \\
			 & \le 2\|Q\eta_s\|\,\|\zeta\| + \|Q\|\,\|\zeta\|^2 + 2\|P\eta_s\|\,\|\nu(\eta,\theta)-\nu(\eta_s,\theta)\| + 2\|P\nu(\eta_s,\theta)\|\,\|\zeta\|         \\
			 & \quad + 2\|P\|\,\|\nu(\eta,\theta)-\nu(\eta_s,\theta)\|\,\|\zeta\| + 2\epsilon M_w\|P\|\,\|\zeta\|                                                \\
			 & \le (2\|Q\eta_s\| + 2\|P\eta_s\|L_\nu + 2\|P\nu(\eta_s,\theta)\| + 2\epsilon M_w\|P\|)\|\zeta\| + (\|Q\| + 2\|P\|L_\nu)\|\zeta\|^2.
		\end{aligned}
	\end{equation*}
	Applying the triangle inequality, the Cauchy--Schwarz inequality, and the Lipschitz condition $\|\nu(\eta,\theta) - \nu(\eta_s,\theta)\| \le L_\nu \|\zeta\|$, we obtain
	\begin{equation*}
		F_V(\eta,\theta) - F_V(\eta_s,\theta) \le C_1(\eta_s)\|\zeta\| + C_2\|\zeta\|^2,
	\end{equation*}
	where $C_1(\eta_s) = 2\|Q\eta_s\| + 2\|P\eta_s\|L_\nu + 2\|P \nu(\eta_s,\theta)\| + 2\epsilon M_w\|P\|$ and $C_2 = \|Q\| + 2\|P\|L_\nu$.

	Therefore, we have
	\begin{equation*}
		F_V(\eta,\theta) \le F_V(\eta_s,\theta) + C_1(\eta_s)\|\zeta\| + C_2\|\zeta\|^2.
	\end{equation*}
	To ensure $F_V(\eta,\theta) \le 0$, it suffices to require
	\begin{equation*}
		C_2\|\zeta\|^2 + C_1(\eta_s)\|\zeta\| + F_V(\eta_s,\theta) \le 0.
	\end{equation*}
	Since $F_V(\eta_s,\theta) < 0$, the polynomial $h(r)=C_2r^2 + C_1(\eta_s)r + F_V(\eta_s,\theta)$ satisfies $h(0)=F_V(\eta_s,\theta)<0$. If $C_2>0$, then $h(r)\to+\infty$ as $r\to\infty$, and $h$ has a unique positive root
	\begin{equation*}
		r_s = \frac{-C_1(\eta_s) + \sqrt{C_1(\eta_s)^2 - 4C_2F_V(\eta_s,\theta)}}{2C_2}.
	\end{equation*}
	If $C_2=0$ and $C_1(\eta_s)>0$, choose
	\begin{equation*}
		r_s = \frac{-F_V(\eta_s,\theta)}{C_1(\eta_s)}.
	\end{equation*}
	If $C_2=0$ and $C_1(\eta_s)=0$, then $h(r)=F_V(\eta_s,\theta)<0$ for all $r\ge 0$, so any positive $r_s$ is admissible. In all cases, for every $\eta\in\mathcal{B}(\eta_s,r_s)\cap\mathcal{X}$, we have $h(\|\zeta\|)\le 0$, which implies $F_V(\eta,\theta)\le 0$. Since $\dot V(\eta)+\rho V(\eta)\le F_V(\eta,\theta)$, the sufficient Lyapunov decrease condition also holds throughout $\mathcal{B}(\eta_s,r_s)\cap\mathcal{X}$.
\end{proof}

\subsection{Proof of Corollary \ref{cor-rad}}\label{prf-rad}

\begin{proof}
	By Theorem \ref{thm-local_radius}, each sampled point $\eta_s \in \mathcal{D}$ admits a certified radius. When $C_2>0$, this radius can be chosen as
	\begin{equation}
		r_s = \frac{-C_1(\eta_s) + \sqrt{C_1(\eta_s)^2 - 4 C_2 F_V(\eta_s,\theta)}}{2 C_2}.
	\end{equation}
	Since $F_V(\eta_s,\theta) \le -\xi < 0$ for all $\eta_s \in \mathcal{D}$,
	\begin{equation}
		r_s \ge \frac{-C_1(\eta_s) + \sqrt{C_1(\eta_s)^2 + 4 C_2 \xi}}{2 C_2}.
	\end{equation}
	Because $\mathcal{X}$ is compact and $\nu(\cdot,\theta)$ is Lipschitz on $\mathcal{X}$, there exist constants $R_\eta,R_\nu < \infty$ such that $\|\eta_s\| \le R_\eta$ and $\|\nu(\eta_s,\theta)\| \le R_\nu$ for all $\eta_s \in \mathcal{X}$. Therefore,
	\begin{equation}
		C_1(\eta_s) = 2\|Q\eta_s\| + 2\|P\eta_s\|L_\nu + 2\|P \nu(\eta_s,\theta)\| + 2\epsilon M_w\|P\|
	\end{equation}
	is uniformly bounded on $\mathcal{X}$. For example, we may take
	\begin{equation}
		\bar{C}_1 = 2\|Q\|R_\eta + 2\|P\|L_\nu R_\eta + 2\|P\|R_\nu + 2\epsilon M_w\|P\|,
	\end{equation}
	so that $C_1(\eta_s) \le \bar{C}_1$ for all $\eta_s \in \mathcal{X}$. Moreover, $C_2 = \|Q\| + 2\|P\|L_\nu$ is independent of $\eta_s$ and satisfies $C_2 \ge 0$.

	If $C_2>0$, define
	\begin{equation}
		g(c) \coloneqq \frac{-c + \sqrt{c^2 + 4 C_2 \xi}}{2 C_2}, \qquad c \ge 0.
	\end{equation}
	Then
	\begin{equation}
		g'(c) = \frac{-1 + c/\sqrt{c^2 + 4 C_2 \xi}}{2 C_2} < 0,
	\end{equation}
	so $g$ is strictly decreasing on $[0,\infty)$. Hence, $C_1(\eta_s) \le \bar{C}_1$ implies $g(C_1(\eta_s)) \ge g(\bar{C}_1)$, and all local stability radii admit the uniform lower bound
	\begin{equation}
		r_{\min} \ge \underline{r} \coloneqq \frac{-\bar{C}_1 + \sqrt{\bar{C}_1^2 + 4 C_2 \xi}}{2 C_2} > 0.
	\end{equation}
	If $C_2=0$ and $\bar{C}_1>0$, then each radius with $C_1(\eta_s)>0$ satisfies
	\begin{equation}
		r_s = \frac{-F_V(\eta_s,\theta)}{C_1(\eta_s)}
		\ge \frac{\xi}{C_1(\eta_s)}
		\ge \frac{\xi}{\bar{C}_1},
	\end{equation}
	while samples with $C_1(\eta_s)=0$ allow any positive radius. Hence, we may take $\underline{r}=\xi/\bar{C}_1>0$ in this case. Finally, if $C_2=0$ and $\bar{C}_1=0$, then $C_1(\eta_s)=0$ for all $\eta_s\in\mathcal{X}$, so any fixed positive radius is admissible for all samples. Therefore a strictly positive uniform lower bound exists in all cases.
\end{proof}

\subsection{Proof of Theorem \ref{thm-pos}}\label{prf-pos}

\begin{proof}
	Since $\mathcal{X}$ is compact, for $\bar r = r_{\min}/2$, there exists a finite $\bar r$-net $\{c_1,\dots,c_{N_{\mathrm{cov}}}\} \subset \mathcal{X}$ such that
	\begin{equation*}
		\mathcal{X} \subset \bigcup_{k=1}^{N_{\mathrm{cov}}} \mathcal{B}(c_k,\bar r).
	\end{equation*}
	By the full-support assumption on $\mu$, each set $\mathcal{B}(c_k,\bar r)\cap\mathcal{X}$ has strictly positive probability. Define
	\begin{equation*}
		q \coloneqq \min_{1 \le k \le N_{\mathrm{cov}}} \mu(\mathcal{B}(c_k,\bar r)\cap\mathcal{X}) > 0.
	\end{equation*}

	For any fixed $k$, the probability that none of the $T$ independent successful samples falls in $\mathcal{B}(c_k,\bar r)\cap\mathcal{X}$ is at most $(1-q)^T$. By the union bound, the probability that at least one of the $N_{\mathrm{cov}}$ covering balls contains no successful sample is at most $N_{\mathrm{cov}}(1-q)^T$.

	Let $\mathcal{E}$ be the event that every $\bar r$-ball $\mathcal{B}(c_k,\bar r)$ contains at least one successful sample. Then
	\begin{equation*}
		\mathbb{P}(\mathcal{E}) \ge 1-N_{\mathrm{cov}}(1-q)^T.
	\end{equation*}

	Conditioning on $\mathcal{E}$, consider any arbitrary point $\bar{\eta} \in \mathcal{X}$. By construction of the $\bar r$-net, there exists a center $c_k$ such that $\|\bar{\eta} - c_k\| \le \bar r = r_{\min}/2$. Since $\mathcal{E}$ holds, there exists a successful sample $\eta_i \in \mathcal{B}(c_k,\bar r)$, so $\|\eta_i - c_k\| \le r_{\min}/2$.

	By the triangle inequality,
	\begin{equation*}
		\|\bar{\eta} - \eta_i\|
		\le \|\bar{\eta} - c_k\| + \|c_k - \eta_i\|
		\le \frac{r_{\min}}{2} + \frac{r_{\min}}{2}
		= r_{\min}.
	\end{equation*}
	Hence, $\bar{\eta} \in \mathcal{B}(\eta_i,r_{\min})\cap\mathcal{X} \subset U_T$. Since $\bar{\eta}$ was arbitrary, $\mathcal{X} \subset U_T$ on the event $\mathcal{E}$.

	Therefore,
	\begin{equation*}
		\mathbb{P}(\mathcal{X} \subset U_T) \ge \mathbb{P}(\mathcal{E}) \ge 1-N_{\mathrm{cov}}(1-q)^T.
	\end{equation*}
	Taking $T \to \infty$ gives
	\begin{equation*}
		\lim_{T \to \infty} \mathbb{P}(\mathcal{X} \subset U_T) = 1,
	\end{equation*}
	which completes the proof.
\end{proof}

\begin{corollary}\label{cor-finite}
	Under the assumptions of Theorem~\ref{thm-pos}, for any $\alpha \in (0,1)$, if 
	\begin{equation}
		T \ge \frac{\log(N_{\mathrm{cov}}/\alpha)}{\log\!\big(1/(1-q)\big)},
	\end{equation}
	then 
	\begin{equation}
		\mathbb{P}(\mathcal{X} \subset U_T) \ge 1-\alpha.
	\end{equation}
\end{corollary}
\begin{proof}
	By Theorem~\ref{thm-pos}, it suffices to require $N_{\mathrm{cov}}(1-q)^T \le \alpha$, which is equivalent to the stated bound on $T$.
\end{proof}

\section{Extension to LMI-Certified Neural Network Controllers}\label{apd-controller-extension}

The main paper develops the two-stage framework for neural network observers. The same training principle applies to neural network controllers whose closed-loop Lyapunov certificates are imposed through LMIs. A representative example is the recurrent neural network (RNN) controller synthesis framework of \cite{gu2022recurrent}, which considers a partially observed discrete-time plant
\begin{equation}
	x_{k+1}=A_Gx_k+B_Gu_k,\qquad y_k=C_Gx_k,
\end{equation}
where the controller only observes $y_k$, not the full state $x_k$. The controller is therefore modeled as a dynamic RNN with hidden state $\xi_k$:
\begin{equation}
	\begin{aligned}
		\xi_{k+1} &= A_K\xi_k+B_{K1}w_k+B_{K2}y_k,\\
		u_k &= C_{K1}\xi_k+D_{K1}w_k+D_{K2}y_k,\\
		v_k &= C_{K2}\xi_k+D_{K3}y_k,\qquad w_k=\phi(v_k),
	\end{aligned}
\end{equation}
where $v_k$ and $w_k$ are the activation input and output, respectively, and $\phi$ is applied elementwise and described by sector or integral quadratic constraint (IQC) conditions.

The key technical step in \cite{gu2022recurrent} is a loop transformation of the activation sector, yielding a transformed controller with parameters $\tilde{\theta}$ and a transformed nonlinearity $\tilde{\phi}$ satisfying a simple sector bound. With the augmented closed-loop state $\zeta_k=[x_k^\top,\xi_k^\top]^\top$, the transformed interconnection can be written abstractly as
\begin{equation}
	\zeta_{k+1}=\mathcal{A}(\tilde{\theta})\zeta_k+\mathcal{B}(\tilde{\theta})z_k,\qquad
	v_k=\mathcal{C}(\tilde{\theta})\zeta_k+\mathcal{D}(\tilde{\theta})z_k,\qquad
	z_k=\tilde{\phi}(v_k).
\end{equation}
Here $z_k$ is the output of the transformed activation $\tilde{\phi}$. Combining the sector constraint for $\tilde{\phi}$ with a quadratic Lyapunov function $V(\zeta)=\zeta^\top P\zeta$ and the S-lemma yields a sequentially convex LMI certificate. For uncertain plants, \cite{gu2022recurrent} augments the nominal dynamics with an IQC description of the uncertainty:
\begin{equation}
	\begin{aligned}
		\zeta_{k+1} &= \mathcal{A}\zeta_k+\mathcal{B}_1 q_k+\mathcal{B}_2 z_k,\\
		v_k &= \mathcal{C}_1\zeta_k+\mathcal{D}_1 q_k+\mathcal{D}_2 z_k,\\
		r_k &= \mathcal{C}_2\zeta_k+\mathcal{D}_3 q_k+\mathcal{D}_4 z_k,
	\end{aligned}
\end{equation}
where $q_k$ is the uncertainty output and $r_k$ is the IQC-filter output. For fixed matrices $\bar P$ and $\bar\Lambda$ from the previous iterate, the robust LMI is
\begin{equation}
	\begin{bmatrix}
		R^\top\Gamma R &
		\begin{bmatrix}
			\mathcal{A} & \mathcal{B}_1 & \mathcal{B}_2 \\
			\mathcal{C}_1 & \mathcal{D}_1 & \mathcal{D}_2
		\end{bmatrix}^{\!\top} \\[0.8em]
		\begin{bmatrix}
			\mathcal{A} & \mathcal{B}_1 & \mathcal{B}_2 \\
			\mathcal{C}_1 & \mathcal{D}_1 & \mathcal{D}_2
		\end{bmatrix} &
		\begin{bmatrix}
			Q_1 & O \\
			O & Q_2
		\end{bmatrix}
	\end{bmatrix}
	\succeq O,
	\label{eq:rnn-controller-lmi}
\end{equation}
where
\begin{equation}
	\Gamma =
	\mathrm{diag}\!\left(
	\rho^2(2\bar P-\bar P^\top Q_1\bar P),
	2\bar\Lambda-\bar\Lambda^\top Q_2\bar\Lambda,
	-M
	\right),\qquad
	R =
	\begin{bmatrix}
		I & O & O\\
		O & O & I\\
		\mathcal{C}_2 & \mathcal{D}_3 & \mathcal{D}_4
	\end{bmatrix}.
\end{equation}
Here $M\in\mathcal{M}$ is the IQC multiplier for the plant uncertainty, with $\mathcal{M}$ denoting the admissible convex multiplier set. This LMI is jointly convex in $(Q_1,Q_2,\tilde{\theta},M)$ for fixed $(\bar P,\bar\Lambda)$, where $Q_1$ and $Q_2$ are inverse-type variables associated with the Lyapunov matrix and sector multiplier. Its feasible set defines the robust convex inner approximation
\begin{equation}
	\mathcal{C}_R(\bar P,\bar\Lambda)
	=
	\{\tilde{\theta}:\exists Q_1,Q_2,\ M\in\mathcal{M}\ \mathrm{s.t.}\ 
	\eqref{eq:rnn-controller-lmi}\},
\end{equation}
which guarantees exponential stability with the prescribed rate. The nominal linear time-invariant (LTI) case uses the analogous set $\mathcal{C}(\bar P,\bar\Lambda)$ obtained by removing the robust IQC block.

To extend our two-stage framework, Stage~I replaces the observer error state $\eta$ with the augmented closed-loop state $\zeta$ and uses sampled Lyapunov decrease as a pre-training signal. For sampled augmented states $\{\zeta_i\}_{i=1}^N$, define the one-step Lyapunov residual
\begin{equation}
	F_{\mathrm{ctrl}}(\zeta,\tilde{\theta})
	=
	V(\zeta^+)-\rho^2 V(\zeta),\qquad
	\zeta^+=f_{\mathrm{cl}}(\zeta;\tilde{\theta}),
\end{equation}
where $0<\rho<1$ is the desired exponential decay rate. The controller can then be trained with a task objective plus a point-guided Lyapunov penalty
\begin{equation}
	\mathcal{L}_{\mathrm{point}}^{\mathrm{ctrl}}(\tilde{\theta})
	=
	\mathcal{L}_{\mathrm{task}}(\tilde{\theta})
	+
	\alpha \frac{1}{N}\sum_{i=1}^{N}
	\big[F_{\mathrm{ctrl}}(\zeta_i,\tilde{\theta})+\delta_{\mathrm{ctrl}}\big]_+
	+
	\gamma\|\tilde{\theta}\|^2,
\end{equation}
where $\delta_{\mathrm{ctrl}}>0$ is the controller Lyapunov margin, $\alpha,\gamma>0$ are loss weights, and $\zeta_i$ may be sampled from rollouts, reset distributions, or a prescribed operating domain.

Stage~II then starts from the pre-trained transformed controller $\tilde{\theta}_{\mathrm{pre}}$ and performs the LMI-certified refinement using the same projection structure as \cite{gu2022recurrent}, for example
\begin{equation}
	\begin{aligned}
		\min_{Q_1,Q_2,\tilde{\theta}}\quad
		& \|\tilde{\theta}-\tilde{\theta}_{\mathrm{pre}}\|_F^2
		+\beta_1\|Q_1-\bar P^{-1}\|_F^2
		+\beta_2\|Q_2-\bar\Lambda^{-1}\|_F^2\\
		\mathrm{s.t.}\quad
		& \eqref{eq:rnn-controller-lmi},
	\end{aligned}
\end{equation}
where $\beta_1,\beta_2>0$ are regularization weights. Alternatively, one may use the eigenvalue-penalty variant from our observer experiments. After a feasible solution is found, one can update $\bar P=Q_1^{-1}$ and $\bar\Lambda=Q_2^{-1}$ and repeat the projection if desired. The final controller inherits the Lyapunov/LMI stability guarantee of \cite{gu2022recurrent}, while the point-guided stage reduces the amount of SDP-based projection or LMI fine-tuning needed for high-capacity RNN controllers.

\section{Additional X-29 Capacity and Robustness Results}\label{sec-exp1}

We use the X-29 aircraft state-space model to evaluate robustness under structural model mismatch and to study the effect of observer capacity. The nominal system is
\begin{equation}
	\begin{cases}
		\dot{x}(t) = Ax(t)+Bu(t)+d(t), \\
		y(t) = Cx(t),
	\end{cases}
\end{equation}
where $x \in \mathbb{R}^4$ and the nominal matrices $A$, $B$, and $C$ are taken from \cite{1992ntrs.rept09932B}. Their numerical values are provided in Appendix~\ref{apd-exp1}.

We compare four observers of increasing capacity: a Linear Baseline, Tiny neural network (Tiny NN), Small NN, and Large NN. A state-feedback controller $u(t)=-K\hat{x}(t)$ is used in all cases, and architectural details are given in Appendix~\ref{apd-exp1}.

To test robustness to structural mismatch between the nominal model and the true system, we perturb the state transition matrix $A$ and the observation matrix $C$ as
\begin{equation}\label{eqn-disAC}
	X' = X + \sigma \frac{Y}{\|Y\|} \|X\|,
\end{equation}
where $X \in \{A, C\}$, $Y$ is a random matrix of the same dimension with entries sampled independently from the standard normal distribution, and $\sigma \in [0, 0.5]$ is the perturbation coefficient. For each $\sigma$, we simulate $20.0$ seconds of closed-loop operation and average over $100$ random seeds. We report two metrics: (i) \emph{success rate}, the fraction of perturbed trials whose trajectories stay within the stability tube $\|x(t)\| \le r_{\mathrm{tube}}$ with $r_{\mathrm{tube}} = 3.0$ m over the full horizon, and (ii) \emph{mean tracking error}, the average Euclidean distance between the actual and desired trajectories over successful trials.

\begin{table}
	\centering
	\caption{\textbf{Quantitative performance comparison under extreme perturbation ($\mathbf{\sigma=0.5}$).} Results are reported as mean $\pm$ standard error of the mean (SEM) across 100 independent trials.}
	\label{tbl-NN_cmp}
	\setlength{\tabcolsep}{7.0pt}
	\footnotesize
	\begin{tabular}{>{\bfseries}lcc||>{\bfseries}lcc}
		\toprule
		\textbf{Architecture} & \textbf{Tracking Error (m)}  & \textbf{Success Rate} & \textbf{Architecture}                 & \textbf{Tracking Error (m)}  & \textbf{Success Rate} \\
		\midrule
		Linear                                & $0.2052 \pm 0.0402$          & 0.49                  & Small NN                              & $0.0558 \pm 0.0312$          & 0.52                  \\
		Tiny NN                               & $0.2056 \pm 0.0402$          & 0.49                  & \cellcolor{green!20}\textbf{Large NN} & \cellcolor{green!20}$\mathbf{0.0118 \pm 0.0028}$ & \cellcolor{green!20}\textbf{0.56} \\
		\bottomrule
	\end{tabular}
\end{table}

\begin{figure}[t]
	\centering
	\includegraphics[width=0.8\linewidth]{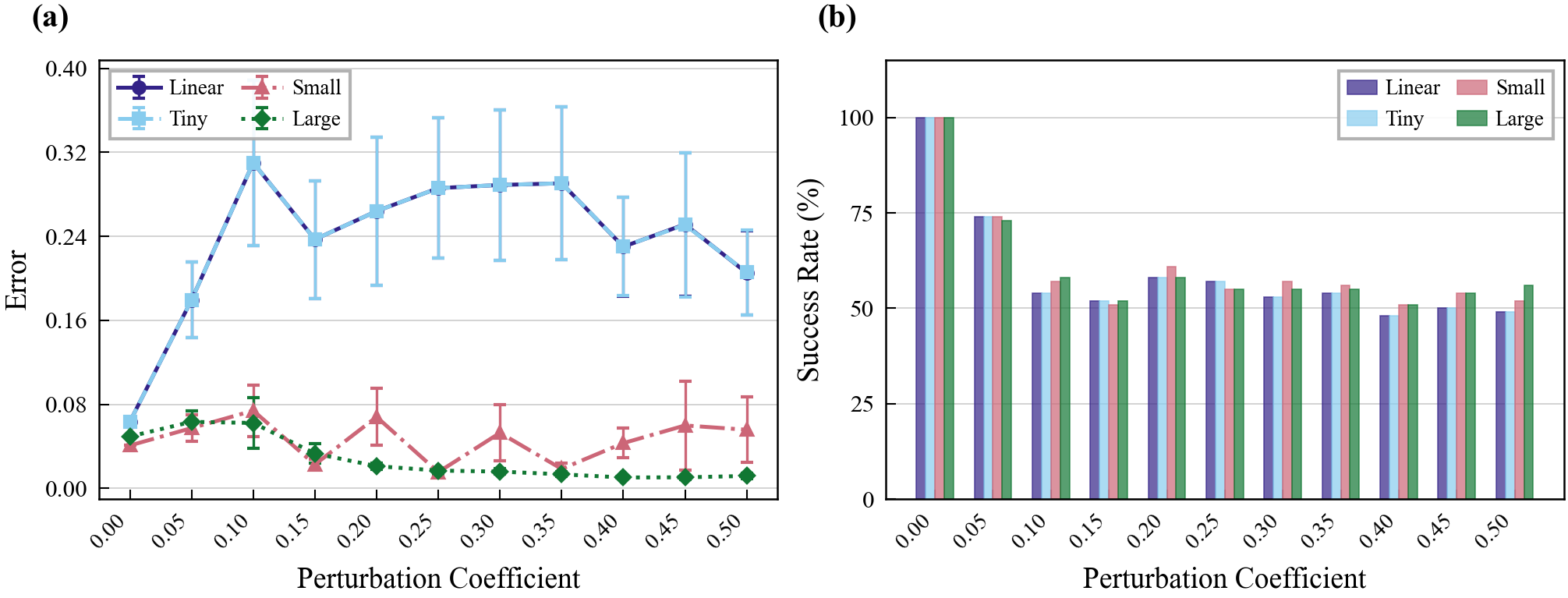}
	\caption{\textbf{Robustness of neural network observers under structural perturbations in the X-29 aircraft benchmark.} Tracking error (top) and success rate (bottom) across different network capacities. Error bars denote the mean $\pm$ SEM across 100 independent trials.}
	\label{fig-NN_cmp}
\end{figure}

Figure~\ref{fig-NN_cmp} and Table~\ref{tbl-NN_cmp} show that robustness improves markedly with network capacity. As $\sigma$ increases, the linear observer and Tiny NN degrade quickly, whereas the Large NN maintains the lowest tracking error and the highest success rate. Under the hardest setting $\sigma=0.5$, the Large NN reduces the tracking error to $0.0118 \pm 0.0028$, compared with $0.2052 \pm 0.0402$ for the linear baseline. Reporting tracking error over successful trials isolates control quality after stabilization, while the success-rate metric accounts for failures separately.

\section{Observer-Error Results}\label{apd-observer-error}

\subsection{Quad-UAV}\label{apd-quad-observer-error}

Table~\ref{tbl-quad-observer-compensation} compares steady-state ground-effect acceleration errors with Neural Lander \cite{shi2019neural}. The neural observer has lower RMSE in both landing and take-off.

\begin{table}[htbp]
	\centering
	\caption{\textbf{Quad-UAV steady-state ground-effect acceleration RMSE.} Mean $\pm$ sample SD over ten seeds per scenario, in $\mathrm{m/s^2}$. The marked entry reports the mean.}
	\label{tbl-quad-observer-compensation}
	\setlength{\tabcolsep}{6pt}
	\footnotesize
	\begin{tabular}{llc}
		\toprule
		Scenario & Method & Steady-state RMSE \\
		\midrule
		Landing & Neural Lander & $1.7295^{\dagger}$ \\
		Landing & Neural observer & $0.02788\pm0.01144$ \\
		Take-off & Neural Lander & $0.1815\pm0.0257$ \\
		Take-off & Neural observer & $(2.588\pm0.6381)\times10^{-7}$ \\
		\bottomrule
	\end{tabular}

	\smallskip
	$\dagger$ Neural Lander landing: median $0.4507$, range $[0.2534,13.3322]$.
\end{table}

\subsection{WaterLily}\label{apd-water-observer-error}

Table~\ref{tbl-water-observer-error} reports steady-state errors in the shared velocity and disturbance channels. The neural observer has lower RMSE in $v$, $\omega$, and $d_v$ than ESO.

\begin{table}[htbp]
	\centering
	\caption{\textbf{WaterLily steady-state observer-error RMSE.} Mean $\pm$ sample SD over five seeds for each observer, in native simulation units.}
	\label{tbl-water-observer-error}
	\setlength{\tabcolsep}{6pt}
	\footnotesize
	\begin{tabular}{lcc}
		\toprule
		Channel & ESO & Neural observer \\
		\midrule
		$v$ & $0.003759\pm0.000318$ & $(3.540\pm1.055)\times10^{-4}$ \\
		$\omega$ & $(6.039\pm0.136)\times10^{-5}$ & $(1.691\pm0.390)\times10^{-5}$ \\
		$d_v$ & $0.12520\pm0.01238$ & $0.04993\pm0.01238$ \\
		$d_\omega$ & $(1.743\pm0.132)\times10^{-6}$ & $(2.363\pm0.621)\times10^{-3}$ \\
		\bottomrule
	\end{tabular}
\end{table}

\section{Experimental Details}\label{apd-exp}

Unless otherwise stated, all neural network observers in our experiments use \texttt{tanh} as the activation function for hidden layers, and both Stage~I pre-training and Stage~II fine-tuning are optimized using the \texttt{Adam} optimizer \cite{kingma2015adam}. For point-guided Lyapunov pre-training, the sampled error states are drawn from a \textbf{truncated normal distribution} supported on the prescribed compact domain $\mathcal{X}$.

\subsection{Experiment: Quad-UAV}\label{apd-exp2}

The baseline control utilizes a PID controller with the following gains:
\begin{equation}
	K_P = \begin{bmatrix} 0.5 \\ 0.5 \\ 0.75 \\ 8.0 \\ 3.0 \\ 3.0 \end{bmatrix}, \quad
	K_I = \begin{bmatrix} 0.0 \\ 0.0 \\ 0.0 \\ 0.0 \\ 0.0 \\ 0.0 \end{bmatrix}, \quad
	K_D = \begin{bmatrix} 0.75 \\ 0.75 \\ 1.0 \\ 1.5 \\ 1.5 \\ 0.75 \end{bmatrix}. \nonumber
\end{equation}
To account for model uncertainties, the ground effect parameters are perturbed to $g_1 = 0.1 \pm 0.01$ and $g_2 = 1.0 \pm 0.25$.

For the Neural Lander baseline, we generated 40,000 data points to learn the dynamics, with training hyperparameters following those specified in \cite{shi2019neural}.

For the proposed Neural Network Observer, the pre-training phase is configured with $\mathrm{epoch} = 100000$ and $\mathrm{lr} = 0.01$, and the fine-tuning phase uses $\mathrm{epoch} = 15000$, $\mathrm{lr} = 0.01$, and $\epsilon = 0.01$. The Lyapunov and LMI hyperparameters are
\begin{equation}
	\begin{aligned}
		P     & = \mathrm{diag}(1,1,1,0.1,0.1,0.1,0.1,0.1,0.1,0.1,0.1,0.1,\\
			  & \qquad\quad\;\; 1,1,1,0.1,0.1,0.1,0.1,0.1,0.1,0.1,0.1,0.1), \\
		\rho  & = 0.1,\quad M_w = 0.5,\quad \xi = 10.0,\quad \beta = 0.01,
	\end{aligned} \nonumber
\end{equation}
and the Stage~II penalty function is
\begin{equation}
	\phi(\lambda) =
	\begin{cases}
		10^5 \lambda, & \lambda > 0, \\
		0,            & \lambda \le 0.
	\end{cases} \nonumber
\end{equation}

\subsection{Experiment: AUV in WaterLily}\label{apd-exp3}

The baseline control framework employs NMPC with a prediction horizon of $H = 20$. We evaluate three specific configurations:

\begin{itemize}[leftmargin=1.5em]
	\item \textbf{Basic NMPC:} A baseline control implementation without disturbance compensation.
	\item \textbf{ESO + NMPC:} NMPC augmented with an ESO for disturbance compensation, using a gain bandwidth of $\omega_0 = 15$.
	\item \textbf{Neural Network Observer + NMPC (Ours):} Our proposed framework. The pre-training phase uses $\mathrm{epoch} = 100000$ and $\mathrm{lr} = 0.01$, while the fine-tuning phase uses $\mathrm{epoch} = 15000$, $\mathrm{lr} = 0.01$, and $\epsilon = 0.01$.
\end{itemize}

For the proposed Neural Network Observer in the AUV experiment, the Lyapunov and LMI hyperparameters are
\begin{equation}
	P = \mathrm{diag}(0.1, 0.1, 0.5, 0.5, 1, 1, 1, 1, 1, 1),\quad \rho = 0.1,\quad M_w = 0.5,\quad \xi = 10.0,\quad \beta = 0.01, \nonumber
\end{equation}
and the Stage~II penalty function is
\begin{equation}
	\phi(\lambda) =
	\begin{cases}
		10^5 \lambda, & \lambda > 0, \\
		0,            & \lambda \le 0.
	\end{cases} \nonumber
\end{equation}

\subsection{Experiment: X-29 Aircraft}\label{apd-exp1}

The nominal matrices used in the X-29 aircraft benchmark are taken from \cite{1992ntrs.rept09932B}:
\begin{equation}
	A =
	\begin{bmatrix}
		-0.4272{\times}10^{-1} & -0.8541{\times}10^{1} & -0.4451              & -0.3216{\times}10^{2} \\
		-0.7881{\times}10^{-3} & -0.5291               & 0.9896               & 0.1439{\times}10^{-9}  \\
		0.4010{\times}10^{-3}  & 0.3542{\times}10^{1}  & -0.2228              & 0.6150{\times}10^{-8}  \\
		0                      & 0                     & 1                    & 0
	\end{bmatrix},
	\nonumber
\end{equation}
\begin{equation}
	B =
	\begin{bmatrix}
		-0.3385{\times}10^{-1} & -0.9386{\times}10^{-1} & 0.4888{\times}10^{-2}  \\
		-0.1028{\times}10^{-2} & -0.1297{\times}10^{-2} & -0.4054{\times}10^{-3} \\
		0.2718{\times}10^{-1}  & -0.5744{\times}10^{-2} & -0.1351{\times}10^{-1} \\
		0                      & 0                      & 0
	\end{bmatrix},
	\nonumber
\end{equation}
\begin{equation}
	C =
	\begin{bmatrix}
		1        & 0     & 0       & 0                  \\
		0        & 57.3  & 0       & 0                  \\
		0        & 0     & 57.3    & 0                  \\
		0        & 0     & 0       & 57.3               \\
		0.007063 & 4.567 & 0.09867 & -0.3809{\times}10^{-4}
	\end{bmatrix}.
	\nonumber
\end{equation}

The baseline control employs a state-feedback control law $u(t) = -K \hat{x}(t)$, where the gain matrix $K$ is given by:
\begin{equation}
	K = \begin{bmatrix}
		-16.99 & 171.31 & 148.15 & 308.07  \\[0.3em]
		-25.65 & -6.74  & -13.01 & 137.50  \\[0.3em]
		5.36   & -91.89 & -76.05 & -130.60
	\end{bmatrix}. \nonumber
\end{equation}
The architectures of the four evaluated observers are detailed below:

\begin{itemize}[leftmargin=1.5em]
	\item \textbf{Linear Baseline:} A conventional observer consisting solely of the linear shortcut $W^{[L+2]}$ (equivalent to zero hidden layers), representing a traditional LTI observer. Parameters are determined via pole-placement.
	\item \textbf{Tiny NN:} A minimal 3-layer architecture with a $[3, 3, 3]$ hidden unit configuration (9 total hidden neurons), solved using pole-placement combined with the \textit{feasp} solver.
	\item \textbf{Small NN:} A medium-capacity 3-layer architecture with a $[32, 64, 32]$ hidden unit configuration (128 total hidden neurons), optimized through pole-placement, the \textit{feasp} solver, and training.
	\item \textbf{Large NN:} An 8-layer deep ResNet with a $[32, 64, 64, 128, 128, 64, 64, 32]$ hidden unit configuration (576 total hidden neurons), optimized using pole-placement, the \textit{feasp} solver, and the proposed training framework.
\end{itemize}

For the pole-placement method, the target poles are assigned at $[-0.5, -1 \pm 0.2i, -1.5 \pm 0.5i, -2 \pm 0.2i, -2.5]$.

In the \textit{feasp} solver configuration, the gain parameter is set to $\epsilon = 0.01$ with a maximum of \texttt{max\_iter} $= 10$ iterations.

For the neural network optimization, the pre-training phase uses $\mathrm{epoch} = 50000$ and $\mathrm{lr} = 0.01$, while the fine-tuning phase uses $\mathrm{epoch} = 10000$, $\mathrm{lr} = 0.01$, and a gain parameter of $\epsilon = 0.01$. For the Lyapunov and LMI objectives, we use
\begin{equation}
	P = I_8,\quad \rho = 0.1,\quad M_w = 0.5,\quad \xi = 10.0,\quad \beta = 0.01, \nonumber
\end{equation}
and the Stage~II penalty function is
\begin{equation}
	\phi(\lambda) =
	\begin{cases}
		10^5 \lambda, & \lambda > 0, \\
		0,            & \lambda \le 0.
	\end{cases} \nonumber
\end{equation}

\subsection{Other Details}\label{apd-expo}

All experiments were conducted on a server equipped with a 12th Gen Intel(R) Core(TM) i5-12600KF CPU, an NVIDIA GeForce RTX 4060 GPU, and 32.0 GB of RAM. The software environment included Windows 10, MATLAB R2024a, Python 3.9 (for the solver), Python 3.8 (for training), and PyTorch 2.4.1.

\newpage
\section*{NeurIPS Paper Checklist}

\begin{enumerate}

	\item {\bf Claims}
	\item[] Question: Do the main claims made in the abstract and introduction accurately reflect the paper's contributions and scope?
	\item[] Answer: \answerYes{}
	\item[] Justification: The abstract and introduction summarize the paper's main contributions, results, and scope.
	\item[] Guidelines:
	      \begin{itemize}
		      \item The answer \answerNA{} means that the abstract and introduction do not include the claims made in the paper.
		      \item The abstract and/or introduction should clearly state the claims made, including the contributions made in the paper and important assumptions and limitations. A \answerNo{} or \answerNA{} answer to this question will not be perceived well by the reviewers.
		      \item The claims made should match theoretical and experimental results, and reflect how much the results can be expected to generalize to other settings.
		      \item It is fine to include aspirational goals as motivation as long as it is clear that these goals are not attained by the paper.
	      \end{itemize}

	\item {\bf Limitations}
	\item[] Question: Does the paper discuss the limitations of the work performed by the authors?
	\item[] Answer: \answerYes{}
	\item[] Justification: The limitations of the work are discussed in Appendix \ref{sec-lim}.
	\item[] Guidelines:
	      \begin{itemize}
		      \item The answer \answerNA{} means that the paper has no limitation while the answer \answerNo{} means that the paper has limitations, but those are not discussed in the paper.
		      \item The authors are encouraged to create a separate ``Limitations'' section in their paper.
		      \item The paper should point out any strong assumptions and how robust the results are to violations of these assumptions (e.g., independence assumptions, noiseless settings, model well-specification, asymptotic approximations only holding locally). The authors should reflect on how these assumptions might be violated in practice and what the implications would be.
		      \item The authors should reflect on the scope of the claims made, e.g., if the approach was only tested on a few datasets or with a few runs. In general, empirical results often depend on implicit assumptions, which should be articulated.
		      \item The authors should reflect on the factors that influence the performance of the approach. For example, a facial recognition algorithm may perform poorly when image resolution is low or images are taken in low lighting. Or a speech-to-text system might not be used reliably to provide closed captions for online lectures because it fails to handle technical jargon.
		      \item The authors should discuss the computational efficiency of the proposed algorithms and how they scale with dataset size.
		      \item If applicable, the authors should discuss possible limitations of their approach to address problems of privacy and fairness.
		      \item While the authors might fear that complete honesty about limitations might be used by reviewers as grounds for rejection, a worse outcome might be that reviewers discover limitations that aren't acknowledged in the paper. The authors should use their best judgment and recognize that individual actions in favor of transparency play an important role in developing norms that preserve the integrity of the community. Reviewers will be specifically instructed to not penalize honesty concerning limitations.
	      \end{itemize}

	\item {\bf Theory assumptions and proofs}
	\item[] Question: For each theoretical result, does the paper provide the full set of assumptions and a complete (and correct) proof?
	\item[] Answer: \answerYes{}
	\item[] Justification: The theoretical results are stated in Section \ref{sec-results}, with complete proofs provided in Appendix \ref{prf-all}.
	\item[] Guidelines:
	      \begin{itemize}
		      \item The answer \answerNA{} means that the paper does not include theoretical results.
		      \item All the theorems, formulas, and proofs in the paper should be numbered and cross-referenced.
		      \item All assumptions should be clearly stated or referenced in the statement of any theorems.
		      \item The proofs can either appear in the main paper or the supplemental material, but if they appear in the supplemental material, the authors are encouraged to provide a short proof sketch to provide intuition.
		      \item Inversely, any informal proof provided in the core of the paper should be complemented by formal proofs provided in appendix or supplemental material.
		      \item Theorems and Lemmas that the proof relies upon should be properly referenced.
	      \end{itemize}

	\item {\bf Experimental result reproducibility}
	\item[] Question: Does the paper fully disclose all the information needed to reproduce the main experimental results of the paper to the extent that it affects the main claims and/or conclusions of the paper (regardless of whether the code and data are provided or not)?
	\item[] Answer: \answerYes{}
	\item[] Justification: The implementation details needed to reproduce the main experimental results are provided in Section \ref{sec-experiments} and Appendix \ref{apd-exp}.
	\item[] Guidelines:
	      \begin{itemize}
		      \item The answer \answerNA{} means that the paper does not include experiments.
		      \item If the paper includes experiments, a \answerNo{} answer to this question will not be perceived well by the reviewers: Making the paper reproducible is important, regardless of whether the code and data are provided or not.
		      \item If the contribution is a dataset and\slash or model, the authors should describe the steps taken to make their results reproducible or verifiable.
		      \item Depending on the contribution, reproducibility can be accomplished in various ways. For example, if the contribution is a novel architecture, describing the architecture fully might suffice, or if the contribution is a specific model and empirical evaluation, it may be necessary to either make it possible for others to replicate the model with the same dataset, or provide access to the model. In general. releasing code and data is often one good way to accomplish this, but reproducibility can also be provided via detailed instructions for how to replicate the results, access to a hosted model (e.g., in the case of a large language model), releasing of a model checkpoint, or other means that are appropriate to the research performed.
		      \item While NeurIPS does not require releasing code, the conference does require all submissions to provide some reasonable avenue for reproducibility, which may depend on the nature of the contribution. For example
		            \begin{enumerate}
			            \item If the contribution is primarily a new algorithm, the paper should make it clear how to reproduce that algorithm.
			            \item If the contribution is primarily a new model architecture, the paper should describe the architecture clearly and fully.
			            \item If the contribution is a new model (e.g., a large language model), then there should either be a way to access this model for reproducing the results or a way to reproduce the model (e.g., with an open-source dataset or instructions for how to construct the dataset).
			            \item We recognize that reproducibility may be tricky in some cases, in which case authors are welcome to describe the particular way they provide for reproducibility. In the case of closed-source models, it may be that access to the model is limited in some way (e.g., to registered users), but it should be possible for other researchers to have some path to reproducing or verifying the results.
		            \end{enumerate}
	      \end{itemize}

	\item {\bf Open access to data and code}
	\item[] Question: Does the paper provide open access to the data and code, with sufficient instructions to faithfully reproduce the main experimental results, as described in supplemental material?
	\item[] Answer: \answerYes{}
	\item[] Justification: The code and reproduction instructions are available at \url{https://github.com/Berry-Myon/LearningNeuralNetworkObserver} and will also be included in the supplementary material.
	\item[] Guidelines:
	      \begin{itemize}
		      \item The answer \answerNA{} means that paper does not include experiments requiring code.
		      \item Please see the NeurIPS code and data submission guidelines (\url{https://neurips.cc/public/guides/CodeSubmissionPolicy}) for more details.
		      \item While we encourage the release of code and data, we understand that this might not be possible, so \answerNo{} is an acceptable answer. Papers cannot be rejected simply for not including code, unless this is central to the contribution (e.g., for a new open-source benchmark).
		      \item The instructions should contain the exact command and environment needed to run to reproduce the results. See the NeurIPS code and data submission guidelines (\url{https://neurips.cc/public/guides/CodeSubmissionPolicy}) for more details.
		      \item The authors should provide instructions on data access and preparation, including how to access the raw data, preprocessed data, intermediate data, and generated data, etc.
		      \item The authors should provide scripts to reproduce all experimental results for the new proposed method and baselines. If only a subset of experiments are reproducible, they should state which ones are omitted from the script and why.
		      \item At submission time, to preserve anonymity, the authors should release anonymized versions (if applicable).
		      \item Providing as much information as possible in supplemental material (appended to the paper) is recommended, but including URLs to data and code is permitted.
	      \end{itemize}

	\item {\bf Experimental setting/details}
	\item[] Question: Does the paper specify all the training and test details (e.g., data splits, hyperparameters, how they were chosen, type of optimizer) necessary to understand the results?
	\item[] Answer: \answerYes{}
	\item[] Justification: The training and test settings are detailed in Section \ref{sec-experiments} and Appendix \ref{apd-exp}.
	\item[] Guidelines:
	      \begin{itemize}
		      \item The answer \answerNA{} means that the paper does not include experiments.
		      \item The experimental setting should be presented in the core of the paper to a level of detail that is necessary to appreciate the results and make sense of them.
		      \item The full details can be provided either with the code, in appendix, or as supplemental material.
	      \end{itemize}

	\item {\bf Experiment statistical significance}
	\item[] Question: Does the paper report error bars suitably and correctly defined or other appropriate information about the statistical significance of the experiments?
	\item[] Answer: \answerYes{}
	\item[] Justification: We report the mean and either the standard deviation (SD) or the standard error of the mean (SEM) for tracking errors over multiple trials with randomized perturbations. These quantities are explicitly indicated in Tables \ref{tbl-NN_cmp}, \ref{tbl-quad}, and \ref{tbl-fluid} and Figures \ref{fig-NN_cmp}, \ref{fig-quad}, and \ref{fig-AUV}.
	\item[] Guidelines:
	      \begin{itemize}
		      \item The answer \answerNA{} means that the paper does not include experiments.
		      \item The authors should answer \answerYes{} if the results are accompanied by error bars, confidence intervals, or statistical significance tests, at least for the experiments that support the main claims of the paper.
		      \item The factors of variability that the error bars are capturing should be clearly stated (for example, train/test split, initialization, random drawing of some parameter, or overall run with given experimental conditions).
		      \item The method for calculating the error bars should be explained (closed form formula, call to a library function, bootstrap, etc.)
		      \item The assumptions made should be given (e.g., Normally distributed errors).
		      \item It should be clear whether the error bar is the standard deviation or the standard error of the mean.
		      \item It is OK to report 1-sigma error bars, but one should state it. The authors should preferably report a 2-sigma error bar than state that they have a 96\% CI, if the hypothesis of Normality of errors is not verified.
		      \item For asymmetric distributions, the authors should be careful not to show in tables or figures symmetric error bars that would yield results that are out of range (e.g., negative error rates).
		      \item If error bars are reported in tables or plots, the authors should explain in the text how they were calculated and reference the corresponding figures or tables in the text.
	      \end{itemize}

	\item {\bf Experiments compute resources}
	\item[] Question: For each experiment, does the paper provide sufficient information on the computer resources (type of compute workers, memory, time of execution) needed to reproduce the experiments?
	\item[] Answer: \answerYes{}
	\item[] Justification: The computational resources used for the experiments are reported in Appendix \ref{apd-expo}.
	\item[] Guidelines:
	      \begin{itemize}
		      \item The answer \answerNA{} means that the paper does not include experiments.
		      \item The paper should indicate the type of compute workers CPU or GPU, internal cluster, or cloud provider, including relevant memory and storage.
		      \item The paper should provide the amount of compute required for each of the individual experimental runs as well as estimate the total compute.
		      \item The paper should disclose whether the full research project required more compute than the experiments reported in the paper (e.g., preliminary or failed experiments that didn't make it into the paper).
	      \end{itemize}

	\item {\bf Code of ethics}
	\item[] Question: Does the research conducted in the paper conforms, in every respect, with the NeurIPS Code of Ethics \url{https://neurips.cc/public/EthicsGuidelines}?
	\item[] Answer: \answerYes{}
	\item[] Justification: The research conducted in this paper conforms to the NeurIPS Code of Ethics.
	\item[] Guidelines:
	      \begin{itemize}
		      \item The answer \answerNA{} means that the authors have not reviewed the NeurIPS Code of Ethics.
		      \item If the authors answer \answerNo, they should explain the special circumstances that require a deviation from the Code of Ethics.
		      \item The authors should make sure to preserve anonymity (e.g., if there is a special consideration due to laws or regulations in their jurisdiction).
	      \end{itemize}

	\item {\bf Broader impacts}
	\item[] Question: Does the paper discuss both potential positive societal impacts and negative societal impacts of the work performed?
	\item[] Answer: \answerNA{}
	\item[] Justification: The paper focuses on foundational mathematical methods for observer stability in dynamical systems. We do not identify a direct path to negative societal applications or significant immediate societal impacts beyond general improvements in control engineering.
	\item[] Guidelines:
	      \begin{itemize}
		      \item The answer \answerNA{} means that there is no societal impact of the work performed.
		      \item If the authors answer \answerNA{} or \answerNo, they should explain why their work has no societal impact or why the paper does not address societal impact.
		      \item Examples of negative societal impacts include potential malicious or unintended uses (e.g., disinformation, generating fake profiles, surveillance), fairness considerations (e.g., deployment of technologies that could make decisions that unfairly impact specific groups), privacy considerations, and security considerations.
		      \item The conference expects that many papers will be foundational research and not tied to particular applications, let alone deployments. However, if there is a direct path to any negative applications, the authors should point it out. For example, it is legitimate to point out that an improvement in the quality of generative models could be used to generate Deepfakes for disinformation. On the other hand, it is not needed to point out that a generic algorithm for optimizing neural networks could enable people to train models that generate Deepfakes faster.
		      \item The authors should consider possible harms that could arise when the technology is being used as intended and functioning correctly, harms that could arise when the technology is being used as intended but gives incorrect results, and harms following from (intentional or unintentional) misuse of the technology.
		      \item If there are negative societal impacts, the authors could also discuss possible mitigation strategies (e.g., gated release of models, providing defenses in addition to attacks, mechanisms for monitoring misuse, mechanisms to monitor how a system learns from feedback over time, improving the efficiency and accessibility of ML).
	      \end{itemize}

	\item {\bf Safeguards}
	\item[] Question: Does the paper describe safeguards that have been put in place for responsible release of data or models that have a high risk for misuse (e.g., pre-trained language models, image generators, or scraped datasets)?
	\item[] Answer: \answerNA{}
	\item[] Justification: The paper poses no such risks.
	\item[] Guidelines:
	      \begin{itemize}
		      \item The answer \answerNA{} means that the paper poses no such risks.
		      \item Released models that have a high risk for misuse or dual-use should be released with necessary safeguards to allow for controlled use of the model, for example by requiring that users adhere to usage guidelines or restrictions to access the model or implementing safety filters.
		      \item Datasets that have been scraped from the Internet could pose safety risks. The authors should describe how they avoided releasing unsafe images.
		      \item We recognize that providing effective safeguards is challenging, and many papers do not require this, but we encourage authors to take this into account and make a best faith effort.
	      \end{itemize}

	\item {\bf Licenses for existing assets}
	\item[] Question: Are the creators or original owners of assets (e.g., code, data, models), used in the paper, properly credited and are the license and terms of use explicitly mentioned and properly respected?
	\item[] Answer: \answerYes{}
	\item[] Justification: All code used for comparison methods is open source and is properly cited and credited in the experiments.
	\item[] Guidelines:
	      \begin{itemize}
		      \item The answer \answerNA{} means that the paper does not use existing assets.
		      \item The authors should cite the original paper that produced the code package or dataset.
		      \item The authors should state which version of the asset is used and, if possible, include a URL.
		      \item The name of the license (e.g., CC-BY 4.0) should be included for each asset.
		      \item For scraped data from a particular source (e.g., website), the copyright and terms of service of that source should be provided.
		      \item If assets are released, the license, copyright information, and terms of use in the package should be provided. For popular datasets, \url{paperswithcode.com/datasets} has curated licenses for some datasets. Their licensing guide can help determine the license of a dataset.
		      \item For existing datasets that are re-packaged, both the original license and the license of the derived asset (if it has changed) should be provided.
		      \item If this information is not available online, the authors are encouraged to reach out to the asset's creators.
	      \end{itemize}

	\item {\bf New assets}
	\item[] Question: Are new assets introduced in the paper well documented and is the documentation provided alongside the assets?
	\item[] Answer: \answerNA{}
	\item[] Justification: The paper does not release new assets.
	\item[] Guidelines:
	      \begin{itemize}
		      \item The answer \answerNA{} means that the paper does not release new assets.
		      \item Researchers should communicate the details of the dataset\slash code\slash model as part of their submissions via structured templates. This includes details about training, license, limitations, etc.
		      \item The paper should discuss whether and how consent was obtained from people whose asset is used.
		      \item At submission time, remember to anonymize your assets (if applicable). You can either create an anonymized URL or include an anonymized zip file.
	      \end{itemize}

	\item {\bf Crowdsourcing and research with human subjects}
	\item[] Question: For crowdsourcing experiments and research with human subjects, does the paper include the full text of instructions given to participants and screenshots, if applicable, as well as details about compensation (if any)?
	\item[] Answer: \answerNA{}
	\item[] Justification: The paper does not involve crowdsourcing nor research with human subjects.
	\item[] Guidelines:
	      \begin{itemize}
		      \item The answer \answerNA{} means that the paper does not involve crowdsourcing nor research with human subjects.
		      \item Including this information in the supplemental material is fine, but if the main contribution of the paper involves human subjects, then as much detail as possible should be included in the main paper.
		      \item According to the NeurIPS Code of Ethics, workers involved in data collection, curation, or other labor should be paid at least the minimum wage in the country of the data collector.
	      \end{itemize}

	\item {\bf Institutional review board (IRB) approvals or equivalent for research with human subjects}
	\item[] Question: Does the paper describe potential risks incurred by study participants, whether such risks were disclosed to the subjects, and whether Institutional Review Board (IRB) approvals (or an equivalent approval/review based on the requirements of your country or institution) were obtained?
	\item[] Answer: \answerNA{}
	\item[] Justification: The paper does not involve crowdsourcing nor research with human subjects.
	\item[] Guidelines:
	      \begin{itemize}
		      \item The answer \answerNA{} means that the paper does not involve crowdsourcing nor research with human subjects.
		      \item Depending on the country in which research is conducted, IRB approval (or equivalent) may be required for any human subjects research. If you obtained IRB approval, you should clearly state this in the paper.
		      \item We recognize that the procedures for this may vary significantly between institutions and locations, and we expect authors to adhere to the NeurIPS Code of Ethics and the guidelines for their institution.
		      \item For initial submissions, do not include any information that would break anonymity (if applicable), such as the institution conducting the review.
	      \end{itemize}

	\item {\bf Declaration of LLM usage}
	\item[] Question: Does the paper describe the usage of LLMs if it is an important, original, or non-standard component of the core methods in this research? Note that if the LLM is used only for writing, editing, or formatting purposes and does \emph{not} impact the core methodology, scientific rigor, or originality of the research, declaration is not required.
	\item[] Answer: \answerNA{}
	\item[] Justification: The core method development in this research does not involve LLMs as important, original, or non-standard components of the method.
	\item[] Guidelines:
	      \begin{itemize}
		      \item The answer \answerNA{} means that the core method development in this research does not involve LLMs as any important, original, or non-standard components.
		      \item Please refer to our LLM policy in the NeurIPS handbook for what should or should not be described.
	      \end{itemize}

\end{enumerate}

\end{document}